\documentclass[twoside,11pt]{article}

\usepackage{jmlr2e}
\usepackage{microtype}
\usepackage{tabularx}
\usepackage{array}

\usepackage{amsmath,amsfonts,bm,bbm,enumitem}
\usepackage{color}

\usepackage{bm}

\usepackage{url}
\usepackage{algorithm,algorithmic,multirow,hhline}
\usepackage{dblfloatfix}
\usepackage{amssymb}
\usepackage{caption}
\usepackage{comment}
\usepackage{subfigure} 
\usepackage{pstricks}
\usepackage{subfloat}
\usepackage{setspace}
\usepackage{hyperref}
\usepackage{color}
\usepackage{booktabs,multirow}
\usepackage{enumitem}
\usepackage{epstopdf}
\input{mysymbol.sty}

\def\ie{\emph{i.e.,\  }}
\def\eg{\emph{e.g.,\  }}

\providecommand{\norm}[1]{\left\|#1\right\|}
\providecommand{\ip}[1]{\boldsymbol{\langle}#1\boldsymbol{\rangle}}

\def\lamb{\boldsymbol{\lambda}}

\usepackage{lastpage}
\jmlrheading{20}{2019}{1-\pageref{LastPage}}{}{}{}{authors}
\ShortHeadings{short title}{authors}
\firstpageno{1}

\begin{document}

\title{GADMM: Fast and Communication Efficient Framework for Distributed Machine Learning}

      
      \author{\name Anis Elgabli \email anis.elgabli@oulu.fi 
      	\AND 
      	\name Jihong Park \email  jihong.park@oulu.fi
      	\AND
      	\name Amrit S. Bedi \email amritbd@iitk.ac.in
      	\AND
      	\name Mehdi Bennis \email mehdi.bennis@oulu.fi
      	\AND
      	\name Vaneet Aggarwal \email vaneet@purdue.edu
      }
\editor{} 
\maketitle

\begin{abstract}
	
When the data is distributed across multiple servers, lowering the communication cost between the servers (or workers) while solving the distributed learning problem is an important problem and is the focus of this paper. In particular, we propose a fast, and communication-efficient decentralized framework to solve the distributed machine learning (DML) problem. The proposed algorithm, Group Alternating Direction Method of Multipliers (GADMM) is based on the Alternating Direction Method of Multipliers (ADMM) framework. The key novelty in GADMM is that it solves the problem in a decentralized topology where at most half of the workers are competing for the limited communication resources at any given time. Moreover, each worker exchanges the locally trained model only with two neighboring workers, thereby training a global model with a lower amount of communication overhead in each exchange. We prove that GADMM converges to the optimal solution for convex loss functions, and numerically show that it converges faster and more communication-efficient than the state-of-the-art communication-efficient algorithms such as the Lazily Aggregated Gradient (LAG) and dual averaging, in linear and logistic regression tasks on synthetic and real datasets. Furthermore, we propose Dynamic GADMM (D-GADMM), a variant of GADMM, and prove its convergence under the time-varying network topology of the workers.
\end{abstract}

\section{Introduction}
\label{sec:intro}

Distributed optimization plays a pivotal role in distributed machine learning applications~\citep{ahmed2013distributed,dean2012large,li2013distributed,li2014communication} that commonly aims to minimize  {$\frac{1}{N}\!\sum_{n=1}^N f_n(\boldsymbol{\Theta})$}~with  $N$~workers. As illustrated in Fig. 1-(a), this problem is often solved by locally minimizing  {$f_n(\boldsymbol{\theta}_n)$} ~at each worker and globally averaging their model parameters  {$\boldsymbol{\theta}_n$}'s (and/or gradients) at a parameter server, thereby yielding the global model parameters~ {$\boldsymbol{\Theta}$} ~\citep{tsianos2012consensus}. Another way is to formulate the problem as an average consensus problem that minimizes  {$\frac{1}{N}\!\sum_{n=1}^N f_n(\boldsymbol{\theta}_n)$} ~under the constraint  {$\boldsymbol{\theta}_n\!=\!\boldsymbol{\Theta},\forall n$} which can be solved using dual decomposition or Alternating Direction Method of Multipliers (ADMM). ADMM is preferable since standard dual decomposition may fail in updating the variables in some cases. For example, if the objective function $f_n(\boldsymbol{\theta}_n)$ is a nonzero affine function of any component in the input parameter ${\boldsymbol \theta}_n$, then the ${\boldsymbol \theta}_n$-update fails, since the Lagrangian is unbounded from below in ${\boldsymbol \theta}_n$ for most choices of the dual variables \citep{boyd2011distributed}. However, using ADMM or dual decomposition, an existence of a central entity is necessary.

Such a centralized solution is, however, not capable of addressing a large network size exceeds the parameter server's coverage range. Even if the parameter server has a link to each worker, communication resources may become the bottleneck since, at every iteration, all workers need to transmit their updated models to the server before the server updates the global model and send it to the workers. Hence, as the number of workers increases, the uplink communication resources become the bottleneck. 
Because of this, we aim to develop a fast and communication-efficient decentralized algorithm, and propose \emph{Group Alternating Direction Method of Multipliers (GADMM)}. GADMM  solves the problem  {$\frac{1}{N}\!\sum_{n=1}^N f_n(\boldsymbol{\theta}_n)$}  subject to  {$\boldsymbol{\theta}_n\!=\!\boldsymbol{\theta}_{n+1},\forall n \in \{1,\cdots,N-1\}$}, in which the workers are divided into two groups ({\it head} and {\it Tail}), and each worker in the head (tail) group communicates only with its two neighboring workers from the tail (head) group as shown in Fig. \ref{fig1}-(b). Due to its communication with only two neighbors rather than all the neighbors or a central entity, the communication in each iteration is significantly reduced. Moreover, by dividing the workers into two equal groups, at most half of the workers are competing for the communication resources at every communication round. 

Despite this sparse communication where each worker communicates with at most two neighbors, we prove that GADMM converges to the optimal solution for convex functions. We numerically show that its communication overhead is lower than that of state-of-the-art communication-efficient centralized and decentralized algorithms including Lazily Aggregated Gradient (LAG)~\citep{chen2018lag}, and dual averaging~\citep{duchi2011dual} for linear and logistic regression on synthetic and real datasets. Furthermore, we propose a variant of GADMM, Dynamic GADMM (D-GADMM), to consider the dynamic networks in which the workers are moving objects (\eg vehicles), so the neighbors of each worker could change over time. Moreover, we prove that D-GADMM inherits the same convergence guarantees of GADMM. Interestingly, we show that D-GADMM not only adjusts to dynamic networks, but it also improves the convergence speed of GADMM, \ie given a static physical topology, keeping on randomly changing the way the connectivity chain is constructed (Fig. \ref{fig1}-(b)) can significantly accelerate the convergence of GADMM. It is worth mentioning that it was shown in~\citep{nedic2018network} as the number of links in the network graph decreases, the convergence speed becomes slower. However, we show that the decrease of the convergence speed of GADMM compared to the standard  parameter server-based ADMM (fully connected graph) due to sparsifying the network graph can be compensated by continuously keep changing neighbors and utilize D-GADMM.

\begin{figure*}
\centering
\includegraphics[width=\textwidth]{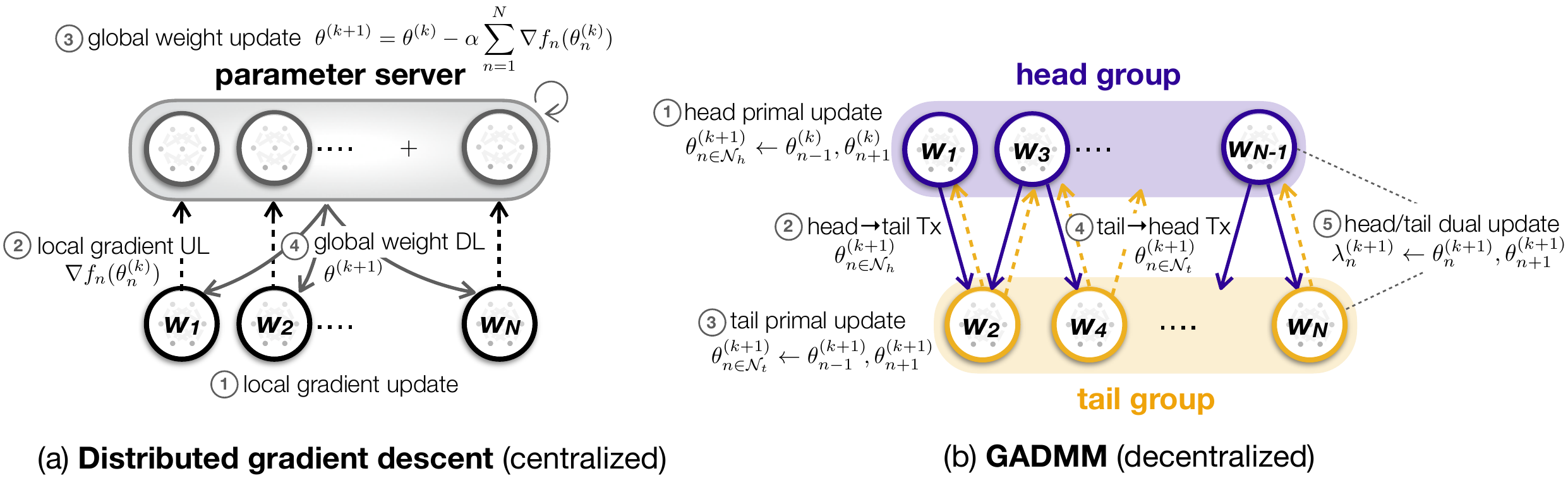}
\caption{  An illustration of (a) distributed gradient descent with a parameter server and (b) GADMM without any central entity.}
\label{fig1}
\end{figure*}

\section{Related Works and Contributions}

\textbf{Distributed Optimization.} There are a variety of distributed optimization algorithms proposed in the literature, such as primal methods \citep{jakovetic2014fast,nedic2014distributed,nedic2009distributed,shi2015proximal} and primal-dual methods \citep{chang2014multi,koppel2017proximity,8624463}. Consensus optimization underlies most of the primal methods, while dual decomposition and ADMM are the most popular among the primal-dual algorithms \citep{glowinski1975approximation,gabay1975dual,boyd2011distributed,jaggi2014communication,ma2017distributed,deng2017parallel}. The performance of distributed optimization algorithms is commonly characterized by their computation time and communication cost. The computation time is determined by the per-iteration complexity of the algorithm. The communication cost is determined by: (i) the number of \emph{communication rounds} until convergence, (ii) the number of \emph{channel uses} per communication round, and (iii) the \emph{bandwidth/power} usage per channel use. Note that the number of communication rounds is proportional to the number of iterations; \eg $2$ rounds at every iteration $k$, for uplink and downlink transmissions in Fig.~1-(a) or for head-to-tail and tail-to-head transmissions in Fig.~1-(b). For a large scale network, the communication cost often becomes dominant compared to the computation time, calling for communication efficient distributed optimization~~\citep{zhang2012communication,Brendan17,park2018wireless,jordan2018communication,liu2019communication,nandan2019}.

\textbf{Communication Efficient Distributed Optimization}.\quad
A vast amount of work is devoted to reducing the aforementioned three communication cost components. To reduce the bandwidth/power usage per channel use, decreasing communication payload sizes is one popular solution, which is enabled by gradient quantization \citep{suresh2017distributed}, model parameter quantization \citep{Zhu:2016aa,nandan2019}, and model output exchange for large-sized models via knowledge distillation~\citep{Jeong18}. To reduce the number of channel uses per communication round, exchanging model updates can be restricted only to the workers whose computation delays are less than a target threshold~\citep{Wang:2018aa}, or to the workers whose updates are sufficiently changed from the preceding updates, with respect to gradients~\citep{chen2018lag}, or model parameters~\citep{liu2019communication}. Albeit their improvement in communication efficiency for every iteration $k$, most of the algorithms in this literature are based on distributed gradient descent, and this limits their required communication rounds to the convergence rate of distributed gradient descent, which is ~ {$\mathcal{O}(1/k)$} for differentiable and smooth objective functions and can be as low as~ {$\mathcal{O}(1/\sqrt{k})$} (\eg when the objective function is non-differentiable everywhere~\citep{boyd2011distributed}). 

On the other hand, primal-dual decomposition methods are shown to be effective in enabling distributed optimization~\citep{jaggi2014communication,boyd2011distributed,ma2017distributed,glowinski1975approximation,gabay1975dual,deng2017parallel}, among which ADMM is a compelling solution that often provides a fast convergence rate with low complexity~\citep{glowinski1975approximation,gabay1975dual,deng2017parallel}. It was shown in~\citep{chen2016direct} that Gauss-Seidel ADMM \citep{glowinski1975approximation} achieves the convergence rate  {$o(1/k)$}. However, this convergence rate is ensured only when the objective function is a sum of two separable convex functions.


Finally, all aforementioned distributed algorithms require  a parameter server being connected to every worker, which may induce a costly communication link to some workers or it may not even be feasible particularly for the workers located beyond the server's coverage. In sharp contrast, we aim at developing a decentralized optimization framework ensuring fast convergence without any central entity.

\textbf{Decentralized Optimization}.\quad
 For decentralized topology, decentralized gradient descent (DGD) has been investigated in \citep{nedic2018network}. Since DGD encounters a lower number of connection per worker compared to parameter-server based GD, it achieves a slower convergence. Beyond GD based approaches, several communication-efficient decentralized algorithms were proposed for both time-variant and invariant topologies. \citep{duchi2011dual,scaman2018optimal} proposed decentralized algorithms to solve the problem for time-invariant topology at a convergence rate of $\mathcal{O}(1/\sqrt{k)}$.
 On the other hand, ~\citep{lan2017communication} proposed a decentralized algorithm that enforces each worker to transmit the updated primal and dual variables at each iteration. Note that, in GADMM, each worker is required to share the primal parameters only per iteration. Finally, it is worth mentioning that a decentralized algorithm was proposed in \citep{he2018cola}, but that algorithm was studied only for linear learning tasks.
 
For time-varying topology, there are a few proposed algorithms in the literature. For instance, \citep{nedic2014distributed} proposed a sub-gradient based algorithm for time-variant directed graph. The algorithm enforces each worker to send two sets of variables to its neighboring nodes per iteration and achieves $\mathcal{O}(1/\sqrt{k)}$ convergence rate. In contrast to that, in D-GADMM, only primal variables are shared with neighbors at each iteration.  Finally, \citep{nedic2017achieving} proposed an algorithm that achieves a linear convergence speed but for strongly convex functions only. Moreover, it also enforces each worker to send more than one set of variables per communication round.

\textbf{Contribution}.\quad
We formulate the decentralized machine learning (DML) problem as a constrained optimization problem that can be solved in a decentralized way. Moreover, we propose a novel algorithm to solve the formulated problem optimally for convex functions. The proposed algorithm is shown to be fast and communication-efficient. It achieves significantly less communication overhead compared to the standard ADMM. The proposed GADMM algorithm allows (i) only half of the workers to transmit their updated parameters at each communication round, (ii) the workers update their model parameters in parallel, while each worker communicates only with two neighbors which makes it communication-efficient. Moreover, we propose D-GADMM which has two advantages: (i) it accounts for time-varying network topology, (ii) it improves the convergence speed of GADMM by  randomly changing neighbors even when the physical topology is not time-varying. Therefore, D-GADMM integrates the communication efficiency of GADMM which uses only two links per worker (sparse graph) with the fast convergence speed of the standard ADMM with parameter server (star topology with N connection to a central entity).
It is worth mentioning that GADMM is closely related to other group-based ADMM methods as in~\citep{Wang:2014,wang2017group}, but these methods consider more communication links per iteration than our proposed GADMM algorithm. Notably, the algorithm in~\citep{wang2017group} still relies on multiple central entities, \ie master workers under a master-slave architecture, whereas GADMM requires no central entity wherein workers are equally divided into head and tail groups.

The rest of the paper is organized as follows. In section \ref{probForm}, we describe the problem formulation. We describe our proposed variant of ADMM (GADMM) and analyze its convergence guarantees in sections \ref{propAlgo} and \ref{propConv} respectively. In section \ref{dgadmm}, we describe D-GADMM which is an extension of our proposed algorithm to time varying networks. In section \ref{eval}, we discuss our simulation results comparing GADMM to the considered baselines. Finally, in section~\ref{conclusion}, we conclude the paper and briefly discuss future directions.

\section{Problem Formulation}
\label{probForm}
We consider a network of $N$ workers where each worker is equipped with the task to learn a global parameter $\boldsymbol{\Theta}$. The aim is to minimize the global convex loss function $F(\boldsymbol{\Theta})$~which is sum of the local convex, proper, and closed functions $f_n(\boldsymbol{\Theta})$ for all $n$. We consider the following optimization problem 
\begin{align}\label{main}
\min_{\boldsymbol{\Theta}} F(\boldsymbol{\Theta}), \quad F(\boldsymbol{\Theta}):= \sum_{n=1}^N f_n(\boldsymbol{\Theta}),
\end{align}
where $\boldsymbol{\Theta} \in \mathbb{R}^d$~is the global model parameter. Gradient descent algorithm can be used to solve the problem in \eqref{main} iterativly in a central entity. The goal here is to solve the problem in a distributed manner. 
The standard technique used in the literature for distributed solution is consensus formulation of \eqref{main} given by. 
\begin{align}
    \min_{\boldsymbol{\Theta},\{\boldsymbol{\theta}_n\}_{n=1}^N}\ \ \ & \sum_{n=1}^N f_n(\boldsymbol{\theta}_n)    \label{orig}\\
\text{s.t.}\ \ &
    \boldsymbol{\theta}_n =\boldsymbol{\Theta}, \ \ \forall \ n.
        \label{orig_c1}
\end{align}
Note that with the reformulation in \eqref{orig}-\eqref{orig_c1}, the objective function becomes separable across the workers and hence  can be solved in a distributed manner. The problem in \eqref{orig}-\eqref{orig_c1} is known as the {\it global consensus problem} since the constraint forces all the variables across different workers to be equal as detailed in \citep{boyd2011distributed}.   The problem in \eqref{orig}-\eqref{orig_c1} can be solved using  the primal-dual based algorithms as in \citep{chang2014distributed,5203882,nedic2009distributed}, saddle point algorithms proposed in \citep{koppel2017proximity,8624463}, and ADMM-based techniques such as \citep{glowinski1975approximation,boyd2011distributed,deng2017parallel}.  ADMM forms an augmented Lagrangian which adds a quadratic term to the Lagrange function and breaks the main problem into sub-problems that are easier to solve per iteration. Note that in the ADMM implementation \citep{boyd2011distributed,deng2017parallel}, only the primal variables $\{\boldsymbol{\theta}_n\}_{n=1}^N$ can be updated in a distributed manner.  However, the step of updating $\boldsymbol{\Theta}$~requires collecting $\boldsymbol{\theta}_n$~from all workers which is communication inefficient \citep{boyd2011distributed}.

The problem formulation in~\eqref{orig}-\eqref{orig_c1} can be solved using standard ADMM (parameter server based-ADMM). The augmented Lagrangian of the optimization problem in~\eqref{orig}-\eqref{orig_c1}  as
\begin{align}
\boldsymbol{\mathcal{L}}_{\rho}(\boldsymbol{\Theta},\{\boldsymbol{\theta}_n\}_{n=1}^N,\boldsymbol{\lambda})=&\sum_{n=1}^N f_n(\boldsymbol{\theta}_n) + \sum_{n=1}^{N} \ip{\lamb _n, \boldsymbol{\theta}_{n} - \boldsymbol{\Theta}}+ \frac{\rho}{2}  \sum_{n=1}^{N} \| \boldsymbol{\theta}_{n} - \boldsymbol{\Theta}\|^2, 
\label{augmentedLagAdmm}
\end{align}
where $\boldsymbol{\lambda}:=[\lamb_1^T, \cdots, \lamb_{N}^T]^T$~is the collection of the dual variables, and $\rho$~is a constant adjusting the penalty for the disagreement between $\bbtheta_n$~and $\boldsymbol{\Theta}$. The primal and dual variables under ADMM are updated in the following three steps. 

\begin{enumerate}[leftmargin=12pt,label={\arabic*)}]
\item At iteration $k+1$, the \emph{primal variable of each workers} is updated as:
\begin{align}
{\boldsymbol{\theta}}_{n}^{k+1} =\arg\min_{\bbtheta_n}\big[f_n(\boldsymbol{\theta}_n) +&\ip{\lamb_{n}^{k}, \boldsymbol{\theta}_{n} - \boldsymbol{\Theta}^{k}}
+ \frac{\rho}{2} \| \boldsymbol{\theta}_{n} - \boldsymbol{\Theta}^{k}\|^2\big], n \in \{1,\cdots, N\}
\label{admmUpdate}
\end{align}

\item After the update in \eqref{admmUpdate}, each workers sends its primal variable (updated model) to the parameter server. The \emph{primal variable of the parameter server} is then updated as:
\begin{align}
{\boldsymbol{\Theta}}^{k+1} = \frac{1}{N}\sum_{n=1}^N\Big(\boldsymbol{\theta}_{n}^{k+1} +\frac{1}{\rho}\lamb_{n}^{k}\Big).
\label{admmServerUpdate}
\end{align}

\item After the update in \eqref{admmServerUpdate}, the parameter server broadcasts its primal variable (the updated global model) to all workers. After receiving the global model ($\boldsymbol{\Theta}^{k+1}$) from the parameter server, \emph{each worker locally updates its dual variable} $\lamb_{n}$~as follows
\begin{align}
{\lamb}_n^{k+1}={\lamb}_n^{k} + \rho({\boldsymbol{\theta}}_{n}^{k+1} - {\boldsymbol{\Theta}}^{k+1}), n=\{1,\cdots,N\}.
\label{admmLambdaUpdateb}
\end{align}
\end{enumerate}
Note that standard ADMM requires a parameter server that collects updates from all workers, update a global model and broadcast that model to all workers. Such a scheme may not be a communication-efficient due to: (i) $N$ workers competing for the limited communication resources at every iteration, (ii) the worker with the weakest communication channel will be the bottleneck for the communication rate of the broadcast channel from the parameter server to the workers, (iii) some workers may not be in the coverage zone of the parameter server. 

In contrast to standard ADMM, we propose a decentralized algorithm that minimizes the communication cost required per worker by allowing only $N/2$ workers to transmit at every communication round, so the communication resources to each worker are doubled compared to parameter server-based ADMM. Moreover, it limits the communication of each worker to include only two neighbors.  We consider the optimization problem in \eqref{orig}-\eqref{orig_c1} and rewrite the constraints as follows.

 \begin{align}
   {\boldsymbol{\theta}^\star}:= \arg\min_{\{\boldsymbol{\theta}_n\}_{n=1}^N} \ &\sum_{n=1}^N f_n(\boldsymbol{\theta}_n)      \label{com_admm}\\
\text{s.t.}\ &
    \boldsymbol{\theta}_{n} = \boldsymbol{\theta}_{n+1}, \ \  n=1,\cdots, N-1.
    \label{com_admm_c1}
\end{align}

Here $\bbtheta^\star$ is the optimal and note that $\boldsymbol{\theta}_{n-1}^\star=\boldsymbol{\theta}_{n}^\star$~and $\boldsymbol{\theta}_{n}^\star=\boldsymbol{\theta}_{n+1}^\star$~for all $n$. This implies that each worker $n$ has joint constraints with only two neighbors (except for the two end workers which have only one). Nonetheless, ensuring $\boldsymbol{\theta}_{n}=\boldsymbol{\theta}_{n+1}$~for all $n \in \{1,\cdots, N-1\}$~at the convergence point yields convergence to a global model parameter that is shared across all workers.

\section{Proposed Algorithm: GADMM}
\label{propAlgo}

We will now describe our proposed algorithm, GADMM, that solves the optimization problem defined in \eqref{com_admm}-\eqref{com_admm_c1} in a decentralized manner. The proposed algorithm is fast since it allows workers belonging to the same group to update their model parameters in parallel, and it is communication-efficient since it allows workers to exchange variables with a minimum number of neighbors and enjoys a fast convergence rate. Moreover, it allows only half of the workers to transmit their updated model parameters at each communication round. Note that when the number of workers who update their parameters per communication round is reduced to half, the communication physical resources (\eg bandwidth) available to each worker are doubled when those resources are shared among workers.  

The \textbf{main idea} of the proposed algorithm is presented in Fig.~\ref{fig1}-(b). The proposed GADMM algorithm splits the network nodes (workers) connected  with a chain into two groups {\it head} and {\it tail} such that each worker in the head's group is connected to other workers through two tail workers.  It allows updating the parameters in parallel for the workers in the same group. In one algorithm iterate, the workers in the head group update their model parameters, and each head worker transmits its updated model to its directly connected tail neighbors. Then, tail workers update their model parameters to complete one iteration. In doing so, each worker (except the edge workers) communicates with only two neighbors to update its parameter, as depicted in~Fig.~\ref{fig1}-(b). Moreover, at any communication round, only half of the workers transmit their parameters, and these parameters are transmitted to only two neighbors.

In contrast  to the Gauss-Seidel ADMM in \citep{boyd2011distributed}, GADMM allows all the head (tail) workers to update their parameters in parallel and still converges to the optimal solution for convex functions as will be shown later in this paper. Moreover, GADMM has much less communication overhead as compared to PJADMM in  \citep{deng2017parallel} which requires all workers to send their parameters to a central entity at every communication round. Also, GADMM has fewer hyperparameters to control and less computation per iteration than PJADMM. The detailed steps of the proposed algorithm are summarized in Algorithm~\ref{alhead}.

To intuitively describe GADMM, without loss of generality, we consider an even $N$ number of workers under their linear connectivity graph shown in Fig.~\ref{fig1}-(b), wherein each head (or tail) worker communicates at most with two neighboring tail (or head) workers, except for the edge workers (\ie first and last workers). With that in mind, we start by writing the augmented Lagrangian of the optimization problem in \eqref{com_admm}-\eqref{com_admm_c1} as
\begin{align}
\boldsymbol{\mathcal{L}}_{\rho}(\{\boldsymbol{\theta}_n\}_{n=1}^N,\boldsymbol{\lambda})=&\sum_{n=1}^N f_n(\boldsymbol{\theta}_n) + \sum_{n=1}^{N-1} \ip{\lamb _n, \boldsymbol{\theta}_{n} - \boldsymbol{\theta}_{n+1}}+ \frac{\rho}{2}  \sum_{n=1}^{N-1} \| \boldsymbol{\theta}_{n} - \boldsymbol{\theta}_{n+1}\|^2, 
\label{augmentedLag4}
\end{align}
Let's divide the $N$~workers into two groups, head ${\cal N}_h=\{\boldsymbol{\theta}_1, \boldsymbol{\theta}_3, \boldsymbol{\theta}_5,\cdots,\boldsymbol{\theta}_{N-1}\}$, and tail ${\cal N}_t=\{\boldsymbol{\theta}_2, \boldsymbol{\theta}_4, \boldsymbol{\theta}_6,\cdots,\boldsymbol{\theta}_{N}\}$, respectively. The primal and dual variables under GADMM are updated in the following three steps. 

\begin{enumerate}[leftmargin=12pt,label={\arabic*)}]
\item At iteration $k+1$, the \emph{primal variables of head workers} are updated as:
\begin{align}
{\boldsymbol{\theta}}_{n}^{k+1} =\arg\min_{\bbtheta_n}\big[f_n(\boldsymbol{\theta}_n) +&\ip{\bblambda_{n-1}^{k}, \boldsymbol{\theta}_{n-1}^{k} - \boldsymbol{\theta}_{n}}+\ip{\lamb_{n}^{k}, \boldsymbol{\theta}_{n} - \boldsymbol{\theta}_{n+1}^{k}} + \frac{\rho}{2} \| \boldsymbol{\theta}_{n-1}^{k} - \boldsymbol{\theta}_{n}\|^2\nonumber
\\
+& \frac{\rho}{2} \| \boldsymbol{\theta}_{n} - \boldsymbol{\theta}_{n+1}^{k}\|^2\big], n \in {\cal N}_h\setminus\{1\}
\label{headUpdate}
\end{align}
Since the first head worker ($n=1$) does not have a left neighbor ($\boldsymbol{\theta}_{n-1}$ is not defined), its model is updated as follows.
\begin{align}
{\boldsymbol{\theta}}_{n}^{k+1} =\arg\min_{\bbtheta_n}\big[f_n(\boldsymbol{\theta}_n)&+\ip{\lamb_{n}^{k}, \boldsymbol{\theta}_{n} - \boldsymbol{\theta}_{n+1}^{k}} 
+\frac{\rho}{2} \| \boldsymbol{\theta}_{n} - \boldsymbol{\theta}_{n+1}^{k}\|^2\big], n=1
\label{headUpdateEdge}
\end{align}

\item After the updates in \eqref{headUpdate} and \eqref{headUpdateEdge}, head workers send their updates to their two tail neighbors. The \emph{primal variables of tail workers} are then updated as:
\begin{align}
{\boldsymbol{\theta}}_{n}^{k+1} =\arg\min_{\boldsymbol{\theta}_n} \big[f_n(\boldsymbol{\theta}_n) +& \ip{\bblambda_{n-1}^{k}, \boldsymbol{\theta}_{n-1}^{k+1}-\boldsymbol{\theta}_{n}}
+\ip{\lamb_{n}^{k}, \boldsymbol{\theta}_{n} - \boldsymbol{\theta}_{n+1}^{k+1}}+ \frac{\rho}{2} \| \boldsymbol{\theta}_{n-1}^{k+1} - \boldsymbol{\theta}_{n}\|^2\nonumber\\ +& \frac{\rho}{2} \| \boldsymbol{\theta}_{n} - \boldsymbol{\theta}_{n+1}^{k+1}\|^2\big], n \in {\cal N}_t\setminus\{N\}.
\label{tailUpdate}
\end{align}
Since the last tail worker ($n=N$) does not have a right neighbor ($\boldsymbol{\theta}_{n+1}$ is not defined), its model is updated as follows.
\begin{align}
{\boldsymbol{\theta}}_{n}^{k+1} =\arg\min_{\boldsymbol{\theta}_n} \big[f_n(\boldsymbol{\theta}_n) +& \ip{\bblambda_{n-1}^{k}, \boldsymbol{\theta}_{n-1}^{k+1}-\boldsymbol{\theta}_{n}}
+ \frac{\rho}{2} \| \boldsymbol{\theta}_{n-1}^{k+1} - \boldsymbol{\theta}_{n}\|^2\big], n=N.
\label{tailUpdateEdge}
\end{align}

\item After receiving the updates from neighbors, \emph{every worker locally updates its dual variables} $\lamb_{n-1}$~and $\lamb_{n}$~as follows
\begin{align}
{\lamb}_n^{k+1}={\lamb}_n^{k} + \rho({\boldsymbol{\theta}}_{n}^{k+1} - {\boldsymbol{\theta}}_{n+1}^{k+1}), n=\{1,\cdots,N-1\}.
\label{lambdaUpdateb}
\end{align}
\end{enumerate}

These three steps of GADMM are summarized in Algorithm \ref{alhead}. We remark that when $f_n(\boldsymbol{\theta}_n)$~is convex, proper, closed, and differentiable for all $n$, the subproblems in \eqref{headUpdate} and \eqref{tailUpdate} are convex and differentiable with respect to $\boldsymbol{\theta}_n$. That is true since the additive terms in the augmented Lagrangian are the addition of quadratic and linear terms, which are also convex and differentiable.

		\begin{algorithm}[t]
			{ 				
			\begin{algorithmic}[1]
					\STATE {\bf Input}: $N, f_n(\boldsymbol{\theta}_n) \ \text{for all} \ n, \rho$
					\STATE {\bf Initialization}:  
					\STATE  ${\cal N}_h=\{\boldsymbol{\theta}_n \mid \text{$n$: odd}\}$,~${\cal N}_t=\{\boldsymbol{\theta}_n \mid \text{$n$: even}\}$ 
					\STATE $\boldsymbol{\theta}_n^{(0)}=0, \lamb_{n}^{(0)}=0$ for all $n$
					\FOR {$k=0,1,2,\cdots,K$}
					
					\STATE \textbf{Head worker $n \in {\cal N}_h$:} 
					\STATE \hspace{.5cm}\textbf{computes} its primal variable $\boldsymbol{\theta}_n^{k+1}$ via \eqref{headUpdate} in parallel; and
					\STATE \hspace{.5cm}\textbf{sends} $\boldsymbol{\theta}_n^{k+1}$ to its neighboring workers $n-1$ and $n+1$. 
				
					\STATE \textbf{Tail worker $n \in {\cal N}_t$:} 
					\STATE \hspace{.5cm}\textbf{computes} its primal variable  $\boldsymbol{\theta}_n^{k+1}$ via \eqref{tailUpdate} in parallel; and
					\STATE \hspace{.5cm}\textbf{sends} $\boldsymbol{\theta}_n^{k+1}$ to its neighbor workers $n-1$ and $n+1$.

					\STATE \textbf{Every worker updates}  the dual variables $\lamb_{n-1}^{k}$ and $\lamb_{n}^{k}$ via \eqref{lambdaUpdateb} locally.
					
					\ENDFOR

				\end{algorithmic}
				\caption{Group ADMM (GADMM) \label{alhead}}
			}						
		\end{algorithm}

\section{Convergence Analysis}
\label{propConv}

In this section, we focus on the convergence analysis of the proposed algorithm. It is essential to prove that the proposed algorithm indeed converges to the optimal solution of the problem in \eqref{com_admm}-\eqref{com_admm_c1} for convex, proper, and closed objective functions. The idea to prove the convergence is related to the proof of Gauss-Seidel ADMM in \citep{boyd2011distributed}, while additionally accounting for the following three challenges: (i) the additional terms that appear when the problem is a sum of more than two separable functions, (ii) the fact that each worker can communicate with two neighbors only, and (iii) the parallel model parameter updates of the head (tail) workers. We show that the GADMM iterates converge to the optimal solution after addressing all the above-mentioned challenges in the proof. Before presenting the main technical Lemmas and Theorems, we start with the necessary and sufficient optimality conditions, which are the primal and the dual feasibility conditions \citep{boyd2011distributed} defined as
\begin{align}\label{primal_feasiblity}
\boldsymbol{\theta}_n^\star =& \boldsymbol{\theta}_{n-1}^\star, n \in \{2,\cdots,N\} \quad \quad\quad\quad\quad\quad \quad \text{(primal feasibility)}
\end{align} 
\begin{align}\label{dual_feasibility}
&\boldsymbol{0} \in \partial f_n(\boldsymbol{\theta}_n^\star) - \bblambda_{n-1}^\star + \lamb_{n}^\star, n \in \{2,\cdots,N-1\}\nonumber\\
&\boldsymbol{0} \in \partial f_n(\boldsymbol{\theta}_n^\star) + \lamb_{n}^\star, n=1
 \quad \quad\quad\quad\quad\quad \quad \quad \quad \text{(dual feasibility)}\nonumber\\
 &\boldsymbol{0} \in \partial f_n(\boldsymbol{\theta}_n^\star) + \lamb_{n-1}^\star, n=N
\quad \quad\quad\quad\quad\quad \quad \quad \quad
\end{align}
We remark that the optimal values $\bbtheta_n^\star$ are equal for each $n$, we denote $\boldsymbol{\theta}^\star=\boldsymbol{\theta}_n^\star = \boldsymbol{\theta}_{n-1}^\star$ for all $n$. Note that, at iteration $k+1$, we calculate $\boldsymbol{\theta}_{n}^{k+1}$ for all $n\in {\cal N}_t\setminus \{N\}$ as in \eqref{tailUpdate}, from the first order optimality condition, it holds that  
\begin{align}
\boldsymbol{0} &\in \partial f_n(\boldsymbol{\theta}_{n}^{k+1}) - \bblambda_{n-1}^{k} + \lamb_{n}^{k} + \rho (\boldsymbol{\theta}_{n}^{k+1}- \boldsymbol{\theta}_{n-1}^{k+1})+ \rho (\boldsymbol{\theta}_{n}^{k+1} - \boldsymbol{\theta}_{n+1}^{k+1}).
\label{eq3}
\end{align}
Next, rewrite the equation in \eqref{eq3} as
\begin{align}
\boldsymbol{0} &\in \partial f_n(\boldsymbol{\theta}_{n}^{k+1}) - \big(\bblambda_{n-1}^{k} + \rho (\boldsymbol{\theta}_{n-1}^{k+1} - \boldsymbol{\theta}_{n}^{k+1})\big)+ \left(\lamb_{n}^{k} + \rho (\boldsymbol{\theta}_{n}^{k+1}- \boldsymbol{\theta}_{n+1}^{k+1})\right).
\label{eq4}
\end{align}
From the update in \eqref{lambdaUpdateb}, the equation in
\eqref{eq4} implies that
\begin{align}
\boldsymbol{0} \in \partial f_n(\boldsymbol{\theta}_{n}^{k+1}) - \lamb_{n-1}^{k+1}+ \lamb_{n}^{k+1}, n\in {\cal N}_t\setminus \{N\}.
\label{slaveEq}
\end{align}
Note that for the $N$-th worker, We calculate $\boldsymbol{\theta}_{N}^{k+1}$ as in \eqref{tailUpdateEdge}, then we follow the same steps, and we get 
\begin{align}
\boldsymbol{0} \in \partial f_n(\boldsymbol{\theta}_{n}^{k+1}) - \lamb_{n-1}^{k+1}, n=N.
\label{slaveEqEdge}
\end{align}
From the result in \eqref{slaveEq} and \eqref{slaveEqEdge}, it holds that the dual feasibility condition in \eqref{dual_feasibility} is always satisfied for all $n  \in {\cal N}_t$. 

Next, consider every $\boldsymbol{\theta}_{n}^{k+1}$ such that $n\in {\cal N}_h\setminus\{1\}$ which is calculated as in \eqref{headUpdate} at iteration $k$. Similarly from the first order optimality condition, we can write 
\begin{align}
\boldsymbol{0} &\in \partial f_n(\boldsymbol{\theta}_{n}^{k+1}) - \bblambda_{n-1}^{k} + \lamb_{n}^{k} + \rho (\boldsymbol{\theta}_{n}^{k+1} - \boldsymbol{\theta}_{n-1}^{k})+ \rho (\boldsymbol{\theta}_{n}^{k+1} - \boldsymbol{\theta}_{n+1}^{k})\label{feasi}.
\end{align}
Note that in \eqref{feasi}, we don't have all the primal variables calculated at $k+1$~instance. Hence, we add and subtract the terms $\boldsymbol{\theta}_{n-1}^{k+1}$ and $\boldsymbol{\theta}_{n+1}^{k+1}$ in \eqref{feasi} to get 

\begin{align}
\boldsymbol{0} \in &\;\partial  f_n(\boldsymbol{\theta}_{n}^{k+1}) - \big(\bblambda_{n-1}^{k}+ \rho(\boldsymbol{\theta}_{n-1}^{k+1} - \boldsymbol{\theta}_{n}^{k+1})\big)+ \left(\lamb_{n}^{k} +\rho(\boldsymbol{\theta}_{n}^{k+1} - \boldsymbol{\theta}_{n+1}^{k+1})\right) \nonumber 
\\
&
+\rho(\boldsymbol{\theta}_{n-1}^{k+1}-\boldsymbol{\theta}_{n-1}^{k})+\rho(\boldsymbol{\theta}_{n+1}^{k+1}-\boldsymbol{\theta}_{n+1}^{k}).
\end{align}
From the update in \eqref{lambdaUpdateb}, it holds that
\begin{align}\label{feasiblity}
\boldsymbol{0} \in \partial & f_n(\boldsymbol{\theta}_{n}^{k+1}) - \lamb_{n-1}^{k+1}+ \lamb_{n}^{k+1}+\rho(\boldsymbol{\theta}_{n-1}^{k+1}-\boldsymbol{\theta}_{n-1}^{k})+\rho(\boldsymbol{\theta}_{n+1}^{k+1}-\boldsymbol{\theta}_{n+1}^{k}).
\end{align}
%
Following the same steps for the first head worker ($n=1$) after excluding the terms $\bblambda_{n-1}^{k}$ and $\rho (\boldsymbol{\theta}_{n}^{k+1} - \boldsymbol{\theta}_{n-1}^{k})$ from $\eqref{feasi}$ (worker $1$ does not have a left neighbor) gives
\begin{align}\label{feasiblityEdge}
\boldsymbol{0} \in \partial & f_n(\boldsymbol{\theta}_{n}^{k+1}) + \lamb_{n}^{k+1}+\rho(\boldsymbol{\theta}_{n+1}^{k+1}-\boldsymbol{\theta}_{n+1}^{k}).
\end{align}
Let $\bbs_{n \in {\cal N}_h}^{k+1}$, the dual residual of worker $n\in {\cal N}_h$~at iteration $k+1$, be defined as follows
\begin{equation}\label{dualResidualEq}
\bbs_{n}^{k+1}
=\left\{\begin{array}{l}
\rho(\boldsymbol{\theta}_{n-1}^{k+1}-\boldsymbol{\theta}_{n-1}^{k})+\rho(\boldsymbol{\theta}_{n+1}^{k+1}-\boldsymbol{\theta}_{n+1}^{k}), \text{ for } n\in {\cal N}_h\setminus\{1\}\\
\rho(\boldsymbol{\theta}_{n+1}^{k+1}-\boldsymbol{\theta}_{n+1}^{k}),  \text{ for }  n = 1.
\end{array}\right.
\end{equation}
Next, we discuss about the primal feasibility condition in \eqref{primal_feasiblity} at iteration $k+1$.  Let $\bbr_{n,n+1}^{k+1}= \boldsymbol{\theta}_{n}^{k+1}-\boldsymbol{\theta}_{n+1}^{k+1}$ be the primal residual of each worker~$n \in \{1,\cdots,N-1\}$. To show the convergence of GADMM, we need to prove that the conditions in \eqref{primal_feasiblity}-\eqref{dual_feasibility} are satisfied for each worker $n$.  We have already shown that the dual feasibility condition in \eqref{dual_feasibility} is always satisfied for the tail workers, and the dual residual of tail workers is always zero. Therefore, to prove the convergence and the optimality of GADMM, we need to show that the $\bbr_{n, n+1}^{k}$ for all $n=1, \cdots, N-1$ and $\bbs_{n \in {\cal N}_h}^{k}$ converge to zero, and $\sum_{n=1}^Nf_n(\boldsymbol{\theta}_n^{k})$~converges to $\sum_{n=1}^Nf_n(\boldsymbol{\theta}^\star)$ as $k\rightarrow\infty$. Now we are in position to introduce our first result in terms of Lemma \ref{lemma:first}. 

\begin{lemma}\label{lemma:first}
For the iterates $\boldsymbol{\theta}_{n}^{k+1}$ generated by Algorithm \ref{alhead}, we have

(i) Upper bound on the optimality gap \begin{align}
&\sum_{n=1}^N[f_n(\boldsymbol{\theta}_{n}^{k+1})-f_n(\boldsymbol{\theta}^\star)]\leq -\sum_{n=1}^{N-1}\ip{\lamb_{n}^{k+1},\bbr_{n,n+1}^{k+1}}+\sum_{n\in {\cal N}_h}\ip{\bbs_{n}^{k+1},\boldsymbol{\theta}_n^\star-\boldsymbol{\theta}_n^{k+1}}.
\label{lem1Eq1}
\end{align}

(ii) Lower bound on the optimality gap	\begin{align}
&\sum_{n=1}^N[f_n(\boldsymbol{\theta}_{n}^{k+1})-f_n(\boldsymbol{\theta}^\star)]\geq -\sum_{n=1}^{N-1}\ip{\lamb_{n}^\star,\bbr_{n,n+1}^{k+1}}.
\label{lem1Eq2}
\end{align}
\end{lemma}
The detailed proof is provided in Appendix \ref{sec:lem1}.  The main idea for the proof is to utilize the optimality of the updates in \eqref{headUpdate} and \eqref{tailUpdate}. We derive the upper bound for the objective function optimality gap in terms of the primal and dual residuals as stated in \eqref{lem1Eq1}. To get the lower bound in \eqref{lem1Eq2} in terms of the primal residual, the definition of the Lagrangian is used at $\rho=0$.  The result in Lemma \ref{lemma:first} is used to derive the main results in  Theorem \ref{theorem} of this paper presented next.

\begin{theorem}\label{theorem}
When $f_n(\boldsymbol{\theta}_n)$ is closed, proper, and convex for all $n$, and the Lagrangian $\boldsymbol{\mathcal{L}}_{0}$ has a saddle point, for GADMM iterates, it holds that 
%
\begin{itemize}
\item[(i)] the primal residual converges to zero as $k\rightarrow\infty$.\ie
\begin{align}
 \lim_{k\rightarrow\infty}\bbr_{n,n+1}^{k}= \boldsymbol{0},  n \in \{1,\cdots, N-1\}. 
\end{align} 
\item[(ii)] the dual residual converges to zero as $k\rightarrow\infty$.\ie
\begin{align}
 \lim_{k\rightarrow\infty} \bbs_{n}^{k}=\boldsymbol{0}, n\in {\cal N}_h.
\end{align}

\item[(iii)] the optimality gap converges to zero as  $k\rightarrow\infty$.\ie
\begin{align}
\lim_{k\rightarrow\infty}\sum_{n=1}^Nf_n(\boldsymbol{\theta}_n^{k})= \sum_{n=1}^Nf_n(\boldsymbol{\theta}^\star).
\end{align}
\end{itemize}

\end{theorem}
\begin{proof}
The detailed proof of Theorem \ref{theorem} is provided in Appendix \ref{sec:them1}. There are three main steps to prove convergence of the proposed algorithm. For a proper, closed, and convex objective function $f_n(\cdot)$, with Lagrangian $\mathcal{L}_0$ which has a saddle point $(\bbtheta^\star,\{\lamb_n\}_{\forall n})$, we define a Lyapunov function $V_k$ as
\begin{align}
V_k =1/\rho\sum_{n=1}^{N-1}\norm{\lamb_n^{k}-\lamb_n^\star}^2 +\rho\sum_{n\in {\cal N}_h\setminus\{1\}}\norm{\boldsymbol{\theta}_{n-1}^{k} -  \boldsymbol{\theta}^\star}^2+\rho\sum_{n\in {\cal N}_h}\norm{\boldsymbol{\theta}_{n+1}^{k} -  \boldsymbol{\theta}^\star}^2.
\label{lyapEq}
\end{align}
In the proof, we show that $V_k$ is monotonically decreasing at each iteration $k$ of the proposed algorithm. This property is then used to prove that the primal residuals go to zero as $k\rightarrow\infty$ which implies that $\bbr_{n,n+1}^k\rightarrow \boldsymbol{0}$ for all $n$. Secondly, we prove that the dual residuals converges to zero as $k\rightarrow\infty$ which implies that $\bbs_{n}^k\rightarrow \boldsymbol{0}$ for all $n\in {\cal N}_h$. Note that the convergence in the first and the second step implies that the overall constraint violation due to the proposed algorithm goes to zero as $k\rightarrow\infty$. In the final step, we utilize statement (i) and (ii) of Theorem \ref{theorem} into the results of Lemma \ref{lemma:first} to prove that the objective optimality gap goes to zero as $k\rightarrow\infty$. 
\end{proof}

\section{Extension to Time-Varying Network: D-GADMM}
\label{dgadmm}

In this section, we present an extension of the proposed GADMM algorithm to the scenario where the set of neighboring workers to each worker is varying over time. Note that the overlay logical topology under consideration is still chain while the physical neighbors are allowed to change. Under this dynamic setting, the execution of the proposed GADMM in Algorithm~\ref{alhead} would be disrupted. Therefore, we propose, D-GADMM (summarized in Algorithm \ref{alg2}) which adjusts to the changes in the  set of neighbors.
	\begin{algorithm}[t]
	{ 
		
		\begin{algorithmic}[1]
			\STATE {\bf Input}: $N, f_n(\cdot)$ for all $n$, $\rho$, and $\tau$
			\STATE {\bf Initialization}:  
			\STATE  ${\cal N}_h=\{\boldsymbol{\theta}_n \mid \text{$n$: odd}\}$,~${\cal N}_t=\{\boldsymbol{\theta}_n \mid \text{$n$: even}\}$ \\
			\STATE $\boldsymbol{\theta}_n^{(0)}=0, \lamb_{n}^{(0)}=0$, for all $n$
			\FOR {$k=0, 1,2,\cdots,K$}
			
			\IF {$k$\text{ mod }$\tau=0$} 
			\STATE \textbf{Every worker:} 
			\STATE \hspace{.5cm}\textbf{broadcasts} current model parameter
			\STATE \hspace{.5cm}\textbf{finds neighbors} and \textbf{refreshes} indices $\{n\}$ as explained in Appendix-D.
			\STATE \hspace{.5cm}\textbf{sends} $\lamb_{n}^{k}$ to its right neighbor (worker $n_{r,k}$)
			\ENDIF

			\STATE \textbf{Head worker $n \in {\cal N}_h$:} 
					\STATE \hspace{.5cm}\textbf{computes} its primal variable $\boldsymbol{\theta}_n^{k+1}$ via \eqref{headUpdate} in parallel; and
					\STATE \hspace{.5cm}\textbf{sends} $\boldsymbol{\theta}_n^{k+1}$ to its neighboring workers $n_{l,k}$ and $n_{r,k}$. 
				
					\STATE \textbf{Tail worker $n \in {\cal N}_t$:} 
					\STATE \hspace{.5cm}\textbf{computes} its primal variable  $\boldsymbol{\theta}_n^{k+1}$ via \eqref{tailUpdate} in parallel; and
					\STATE \hspace{.5cm}\textbf{sends} $\boldsymbol{\theta}_n^{k+1}$ to its neighbor workers $n_{l,k}$ and $n_{r,k}$.

					\STATE \textbf{Every worker updates}  the dual variables $\lamb_{n-1}^{k}$ and $\lamb_{n}^{k}$ via \eqref{lambdaUpdateb} locally.			
			
			\ENDFOR
		\end{algorithmic}
		\caption{Dynamic GADMM (D-GADMM)\label{alg2}}
	}						
\end{algorithm}

 In D-GADMM, all the workers periodically reconsider their connections after every $\tau$ iterations. if neighbors and/or worker assignment to head/tail group change, every worker broadcasts its current model parameter to the new neighbors. We assume that the workers run an algorithm that can keep constructing a communication-efficient logical chain as the underlying physical topology changes, and the design of such an algorithm is not the main focus of the paper.  It is worth mentioning that a logical graph that starts at one worker and reaches every other worker only once in the most communication-efficient way is an NP-hard problem. It can be easily shown that this problem can be reduced to a Traveling Salesman Problem (TSP). This is due to the fact that starting from one worker and choosing every next one such that the total communication cost is minimized is exactly equal to starting from one city and reaching every other city such that the total distance is minimized, \ie the workers in our problem are the cities in TSP, and the communication cost between each pair of workers in our problem is the distance between each pair of cities in TSP. Hence, proposed heuristics to solve TSP \citep{lenstra1975some,bonomi1984n} can be used to construct the chain in our problem with the aid of a central entity, and then the algorithm continues working on the decentralized way. Decentralized heuristics for TSP have been proposed which can also be used \citep{peterson1990parallel,dorigo1997ant}.  However, in this paper, we use a simple decentralized heuristic that we describe in Appendix \ref{chainConst}. Finally, it is worth mentioning that D-GADMM can still be utilized even if the physical topology does not change. In such a scenario, the workers can agree on a predefined sequence of logical chains, so changing neighbors does not require running an online algorithm, and thus it encounters zero overhead. We observe in section \ref{eval} that D-GADMM can improve the convergence speed when it is utilized even when the physical topology does not change.

The detailed steps of D-GADMM is described in Algorithm \ref{alg2}. We note that before the execution, we assume that all the nodes are connected to each other with a chain. Each node is associated with an index $n$ and there exists a link from node $1$ to node $2$, node $2$ to node $3$, and so on till $N-1$ to $N$. For each node $n$ there is an associated primal variable $\boldsymbol{\theta}_n$ from $n=1$ to $N$ and a dual variable $\boldsymbol{\lambda}_n$ from node $n=1$ to $N-1$. Under the dynamic settings, we assume that the nodes at position $n=1$ and $N$ are fixed while the other nodes are allowed to move in the network. This means that instead of having the connection in the order $1-2-3\cdots N$, the nodes are allowed to connect in any order such as $1-5-3-\cdots-4-N$, or $1-4-2-\cdots-5-N$, etc. Alternatively, the neighbors of each node $n$ are no longer fixed under the dynamic settings. To denote this behavior, since the topology is still the chain, we call the left neighbor to node $n$ as $n_{l,k}$ at iteration $k$ and similarly
$n_{r,k}$ for the right neighbor node. Therefore, at each iteration $k$, each node implements the algorithm considering {$n_{l,k}$} and {$n_{r,k}$} as its neighbors. Note that, when the topology changes at iteration $k$, every worker $n$ transmits its right dual variable $\boldsymbol{\lambda}_n^k$ to its right neighbor in the new chain to ensure that both neighbors share the same dual variable. Therefore, the right neighbor of each worker $n$ will replace $\bblambda_{n_{l,k}}^{k}$ with the dual variable that is received from its new left neighbor. With that, we show in Appendix \ref{DGADMMproof} that the algorithm  converges to the optimal solution in a similar manner to GADMM.

\section{Numerical Results}
\label{eval}
\begin{table*}[t]{
\centering 
\resizebox{
\textwidth}{!}{
\begin{tabular}{c c c c c c c c c } 
    \toprule
   \multirow{2}[4]{*}{ \textbf{Iteration}} & \multicolumn{4}{c}{\textbf{Linear} Regression}&\multicolumn{4}{c}{\textbf{Logistic} Regression} \\
    & $N=$14 & 20 & 24 & 26 & $N=$14 & 20 & 24 & 26\\ 
    \cmidrule(r){1-1}\cmidrule(l){2-5}   \cmidrule(l){6-9}
    \multicolumn{1}{l}{LAG-PS} & 542& 8,043& 54,249& 141,132 &  21,183 & 20,038 & 19,871 & 20,544\\
    \multicolumn{1}{l}{LAG-WK}&385 & 6,444 &  44,933& 121,134 &  18,584 & 17,475 & 17,050 & 17,477\\
    \multicolumn{1}{l}{ \textbf{GADMM}}&\textbf{78} & \textbf{292}&  \textbf{558}& \textbf{550}  & \textbf{120} &  \textbf{235} & \textbf{112} & \textbf{160} \\
    \multicolumn{1}{l}{GD} & 524 & 8,163& 55,174& 143,651 &  1,190 & 1,204 & 1,181 & 1,152\\
    \bottomrule\\
    \toprule
   \multirow{2}[4]{*}{ \textbf{TC}} & \multicolumn{4}{c}{\textbf{Linear} Regression}&\multicolumn{4}{c}{\textbf{Logistic} Regression} \\ 
    & $N=$14  & 20 & 24 & 26 & $N=$14 & 20 & 24 & 26\\ 
    \cmidrule(r){1-1}\cmidrule(l){2-5}   \cmidrule(l){6-9}
    \multicolumn{1}{l}{LAG-PS} & 3,183 & 52,396& 363,571& 1,035,778 & 316,570 & 419,819 & 495,792 & 553,493\\
    \multicolumn{1}{l}{LAG-WK}& \textbf{820} & 12,369 & 82,985& 241,944 & 18,786 & 17,835 & 17,432 & 17,915\\
    \multicolumn{1}{l}{ \textbf{GADMM}}& 1,092& \textbf{5,840}&  \textbf{13,392}& \textbf{14,300}  &  \textbf{696} & \textbf{1,962} & \textbf{1,030} & \textbf{1,712}\\
    \multicolumn{1}{l}{GD} & 7,860 & 171,423& 1,379,350& 3,878,577 & 17,850 & 25,284 & 29,525 & 31,104\\
    \bottomrule
\end{tabular}
}
\caption{ The required number of iterations (top) and total communication cost (bottom) to achieve the target objective error $10^{-4}$ for different number of workers, in linear and logistic regression with the real datasets. }\label{tab:1}}
\end{table*}

To validate our theoretical foundations, we numerically evaluate the performance of GADMM in linear and logistic regression tasks, compared with the following benchmark algorithms.
\begin{itemize}[leftmargin=12pt]
\item \textbf{LAG-PS}~\citep{chen2018lag}: A version of LAG where parameter server selects communicating workers. 
\item \textbf{LAG-WK}~\citep{chen2018lag}: A version of LAG where workers determine when to communicate with the server.
\item \textbf{Cycle-IAG}~\citep{blatt2007convergent,gurbuzbalaban2017convergence}: A cyclic modified version of the incremental aggregated gradient (IAG).
\item \textbf{R-IAG}~\citep{chen2018lag,Schmidt:17}: A non-uniform sampling version of stochastic average gradient (SAG).
\item \textbf{GD}: Batch gradient descent.
\item \textbf{DGD}~\citep{nedic2018network} Decentralized gradient descent.
\item \textbf{DualAvg}~\citep{duchi2011dual} Dual averaging.
\end{itemize}

For the tuning parameters, we use the setup in \citep{chen2018lag}. For our decentralized algorithm, we consider $N$ workers without any central entity, whereas for centralized algorithms, a uniformly randomly selected worker is considered as a central controller having a direct link to each worker. The performance of each algorithm is measured using:

\begin{itemize}
\item the \textbf{objective error} $|\sum_{n=1}^N \big[f_n(\boldsymbol{\theta}_n^{(k)})-f_n(\boldsymbol{\theta}^*)\big] |$\normalsize~at iteration $k$.

\item (ii) The \textbf{total communication cost~(TC)}. The TC of a decentralized algorithm is $\sum_{t=1}^{T_a} \sum_{n=1}^N {\bf 1}_{n,t} \cdot L_{n,t}^m$\normalsize, where $T_a$\normalsize~is the number of iterations to achieve a target accuracy $a$\normalsize, and ${\bf 1}_{n,t}$\normalsize~denotes an indicator function that equals~$1$ if worker $n$\normalsize~is sending an update at~$t$\normalsize, and $0$ otherwise. The term $L_{n,t}^m$\normalsize~is the cost of the communication link between workers $n$\normalsize~and $m$\normalsize~at communication round $t$\normalsize. Next, let $L_{n,t}^c$\normalsize~denote the cost of the communication between worker $n$ and the central controller at $t$. Then, the TC of a centralized algorithm is $\sum_{t=1}^{T_a} ( L_{\text{BC},t}^c +  \sum_{n=1}^N  {\bf 1}_{{n,t}} \cdot L_{n,t}^c )$\normalsize, where $L_{\text{BC},t}^c$ and $L_{n,t}^c$\normalsize's correspond to downlink broadcast and uplink unicast costs, respectively. It is noted that the communication overhead in~\citep{chen2018lag} only takes into account uplink costs.

\item The total running time (clock time) to achieve objective error $a$. This metric considers both the communication and the local computation time. We consider $L_{n,t}^m = L_{n,t}^c=L_{\text{BC},t}^c=1$ unless otherwise specified.
\end{itemize}

All simulations are conducted using the synthetic and real datasets described in \citep{Dua:2019,chen2018lag}. The synthetic data for the linear and logistic regression tasks are generated as described in~\citep{chen2018lag}. We consider $1,\!200$ samples with $50$ features, which are evenly split into workers. Next, the real data tests linear and logistic regression tasks with \textbf{Body Fat} ($252$ samples, $14$ features) and \textbf{Derm} ($358$ samples, $34$ features) datasets~\citep{Dua:2019}, respectively. As the real dataset is smaller than the synthetic dataset, we by default consider $10$ and $24$ workers for the real and synthetic datasets, respectively.

 \begin{figure}
\centering
\includegraphics[trim=0.1in 0in 0.1in 0.1in, clip,  width=\textwidth]{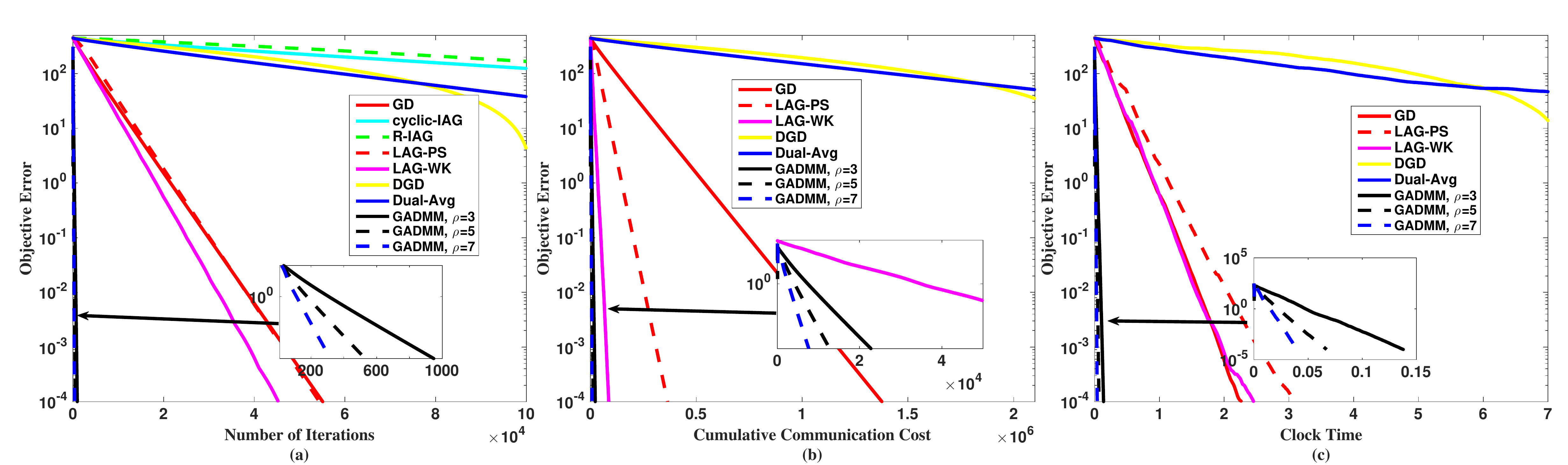}
\caption{ Objective error, total communication cost, and total running time comparison between GADMM and five benchmark algorithms, in \emph{linear} regression with synthetic ($N=24$) datasets.}
\label{comp_linSynth}
\end{figure}

\begin{figure}
\centering
\includegraphics[trim=0.1in 0in 0.1in 0.1in, clip,  width=\textwidth]{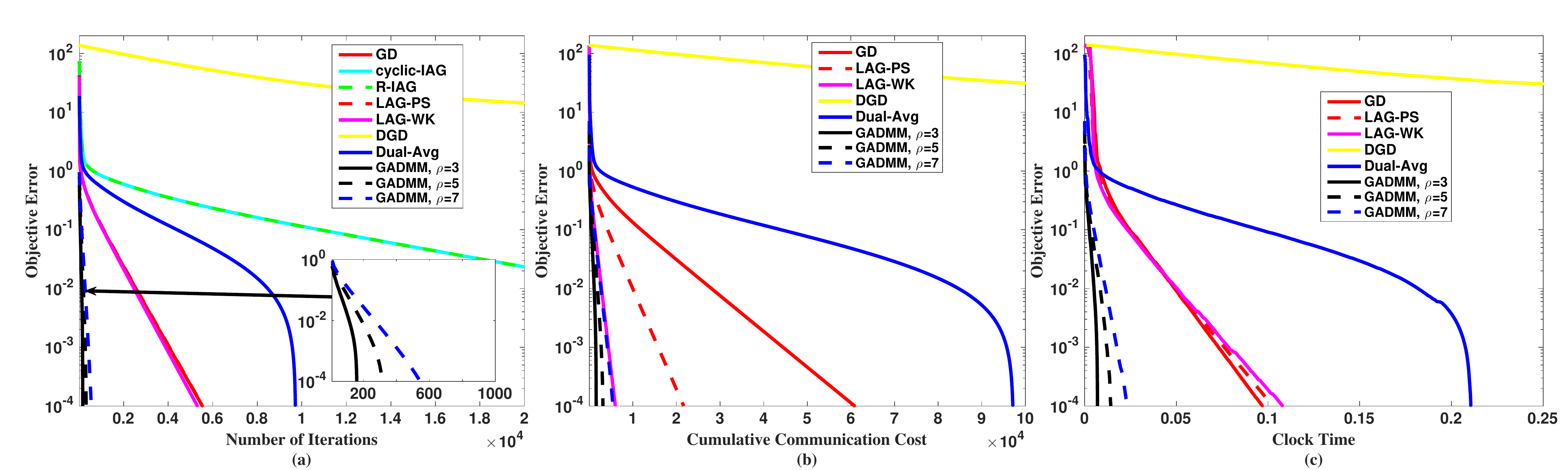}
\caption{ Objective error, total communication cost, and total running time comparison between GADMM and five benchmark algorithms, in \emph{linear} regression with real ($N=10$) datasets.}
\label{comp_linReal}
\end{figure}

Figs.~\ref{comp_linSynth}, \ref{comp_linReal}, \ref{comp_logSynth}, and \ref{comp_logReal} corroborate that GADMM outperforms the benchmark algorithms by several orders of magnitudes, thanks to the idea of two alternating groups where each worker communicates only with two neighbors. For linear regression with the synthetic dataset, Fig.~2 shows that all variants of GADMM with $\rho=$\normalsize$\;3,5,$ and $7$ achieve the target objective error of $10^{-4}$ in less than $1,\!000$ iterations, whereas GD, LAG-PS, and LAG-WK (the closest among baselines) require more than $40,\!000$ iterations to achieve the same target error. Furthermore, the TC of GADMM with $\rho=$\normalsize$\;3$ and $\rho=$\normalsize$\;5$~are $6$ and $9$ times lower than that of LAG-WK respectively. Table~1 shows similar results for different numbers of workers, only except for linear regression with the smallest number of workers ($14$), in which LAG-WK achieves the lowest TC. We also observe from Figs.~\ref{comp_linSynth} and \ref{comp_linReal} that GADMM outperforms all baselines in terms of the total running time, thanks to the fast convergence. GADMM performs matrix inversion which is computationally complex compared to calculating gradient. However, the computation cost per iteration is compensated by fast convergence.

For logistic regression, Figs.~\ref{comp_logSynth} and \ref{comp_logReal} validate that GADMM outperforms the benchmark algorithms, as in the case of linear regression in Figs.~\ref{comp_linSynth} and \ref{comp_linReal}. One thing that is worth mentioning here is shown in Fig~\ref{comp_logSynth}-(c), where we can see that the total running time of GADMM is equal to the running time of GD. The reason behind this is that the logistic regression problem is not solved in a closed-form expression at each iteration. However, GADMM still significantly outperforms GD in communication-efficiency.

Next, comparing the results in Fig.~2 and Fig.~3, we observe that the optimal $\rho$\normalsize~depends on the data distribution across workers. Namely, when the local data samples of each worker are highly correlated with the other workers' samples (\ie Body Fat dataset, Fig.~3), the local optimal of each worker is very close to the global optimal. Therefore, reducing the penalty for the disagreement between $\theta_n$\normalsize~and $\theta_{n+1}$\normalsize~by lowering $\rho$\normalsize~yields faster convergence. Following the same reasoning, higher $\rho$\normalsize~provides faster convergence when the local data samples are independent of each other (\ie synthetic datasets in Fig.~2).

\begin{figure}
\centering
\includegraphics[trim=0.1in 0in 0.1in 0.1in, clip,  width=\textwidth]{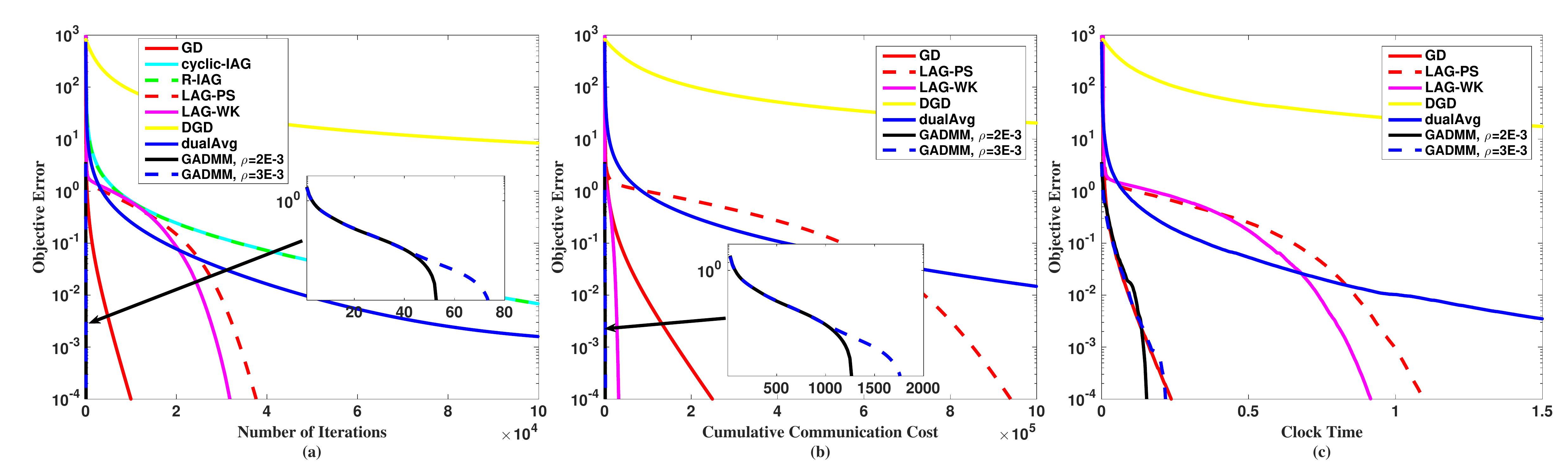}
\caption{ Objective error, total communication cost, and total running time comparison between GADMM and five benchmark algorithms, in \emph{logistic} regression with synthetic ($N=24$) datasets.}
\label{comp_logSynth}
\end{figure}

\begin{figure}
\centering
\includegraphics[trim=0.1in 0in 0.1in 0.1in, clip,  width=\textwidth]{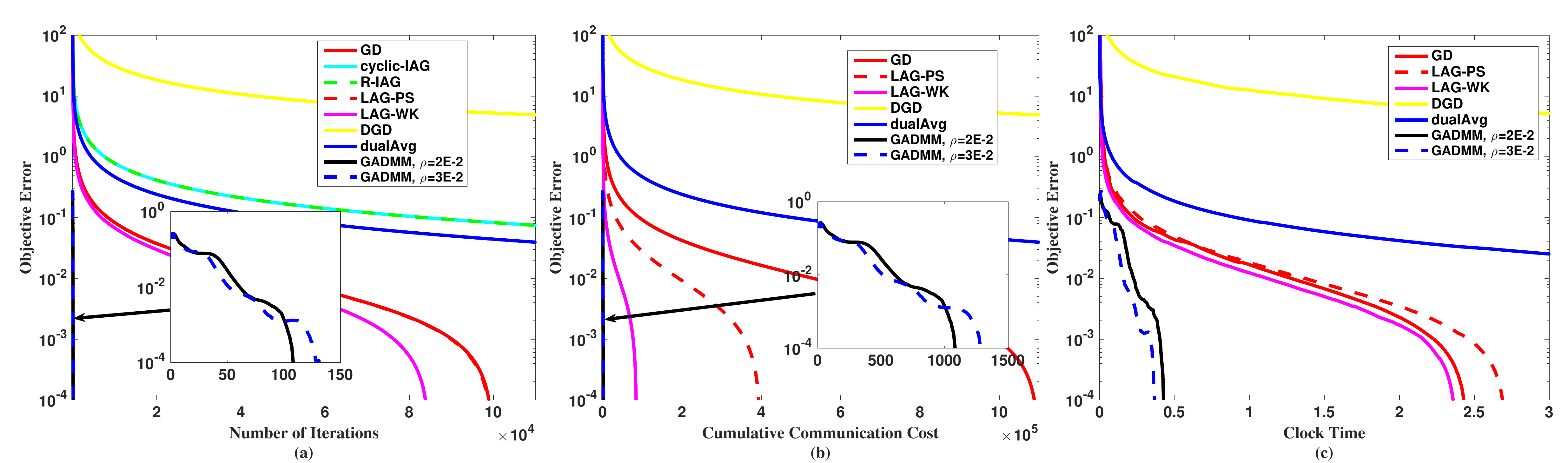}
\caption{ Objective error, total communication cost, and total running time comparison between GADMM and five benchmark algorithms, in \emph{logistic} regression with real ($N=10$) datasets.}
\label{comp_logReal}
\end{figure}

Fig.~\ref{gadmm_evalEx}-(a) and (b) demonstrate that GADMM is communication efficient under different network topologies. In fact, the TC calculations of GADMM in Table~1 and Fig.~\ref{comp_linSynth} rely on a unit communication cost for all communication links, i.e., $L_{n,t}^m = L_{n,t}^c=L_{\text{BC},t}^c=1$, which may not capture the communication efficiency of GADMM under a generic network topology. Instead, we use the consumed energy per communication iteration as the communication cost metric. We illustrate the cumulative distribution function (CDF) of TC by observing $1,\!000$ different network topologies. At the beginning of each observation, $24$ workers are randomly distributed over a $10\!\times\!10$ m$^2$ square area. In GADMM, the method described in Appendix \ref{chainConst} is used to construct the logical chain.  In centralized algorithms, the worker closest to the center becomes a central worker associating with all the other workers. We assume that the bandwidth is evenly distributed among users, and we also assume that each worker needs a bit rate of $10$Mbps to transmit its model in a one-time slot. Therefore, the communication cost per worker per iteration is the amount of energy that worker consumes to achieve the rate of $10$Mbps. Note that according to Shannon's formula, the achievable rate is a function of the bandwidth and power, \ie $R=B\cdot log_2(\frac{P}{d^2\cdot N_0\cdot B})$, where $B$ is the bandwidth, $P$ is the communication power, $N_0$ is the noise spectral density, and $d$ is the distance between the transmitter and the receiver \citep{mckeague1981capacity},  so we assume a free-space communication link. In our simulations, we assume, $B=2$MHz, $N_0=1E-6$, we find the required power (energy) to achieve $10$Mbps over link $l$ at time slot $t$, and that reflects the communication cost of using link $l$ at time slot $t$.


The CDF results in Fig.~\ref{gadmm_evalEx}-(a) and (b) show that with high probability, GADMM achieves much lower TC in both linear and logistic regression tasks for generic network topology, compared to other baseline algorithms. On the other hand, Fig.~\ref{gadmm_evalEx}-(c) validates that GADMM guarantees  consensus on the model parameters of all workers when training converges. Indeed, GADMM complies with the constraint $\boldsymbol{\theta}_n=\boldsymbol{\theta}_{n+1}$ in~\eqref{orig_c1}. We observe in Fig.~4-(c) that the average consensus constraint violation (ACV), defined as ${\sum_{n=1}^{N-1}\! |\boldsymbol{\theta}_n^{(k)} - \boldsymbol{\theta}_{n+1}^{(k)}|/N}$, goes to zero with the number of iterations. Specifically, AVC becomes $8\!\times\!10^{-7}$ after $495$ iterations at which the loss becomes $1\!\times\! 10^{-4}$. This underpins that GADMM is robust against its consensus violations temporarily at the early phase of training, thereby achieving the average consensus at the~end.

\begin{figure*}
\centering
\includegraphics[trim=0in 0in 0in 0in, clip, width=\textwidth, height=4cm]{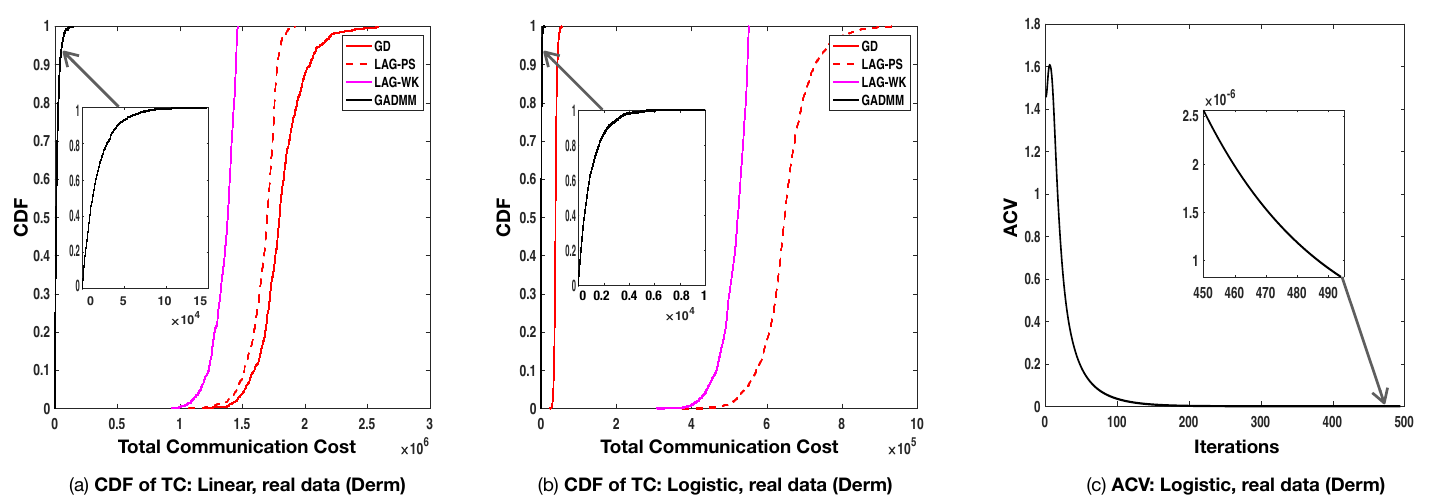}
\caption{The cumulative distribution function (CDF) of total communication cost (TC) in (a) linear and (b) logistic regression by uniformly randomly distributed $24$ workers with $1,\!000$ observations, and (c) the average consensus constraint violation (ACV) of GADMM in logistic regression by $4$ workers.}
\label{gadmm_evalEx}
\end{figure*}

We now extend GADMM to D-GADMM, and evaluate its performance under the time-varying network topology. One note to make, in simulating D-GADMM, we do not exchange dual variables between neighbors at every topology change as described in line 10, Algorithm \ref{alg2}. However, as we will show, D-GADMM still converges. Therefore, the extra communication overhead that might be encountered in D-GADMM when workers share their dual variables is avoided and the convergence is still preserved. We change the topology every $15$ iterations. Therefore, we assume that the system coherence time is 15 iterations. To simulate the change in the topology, $50$ workers are randomly distributed over a $250\!\times\!250$ m$^2$ square area every $15$-th iteration. D-GADMM uses the method described in appendix D which consumes 2 iterations (4 communication rounds) to build the chain. In contrast, GADMM keeps the logical worker connectivity graph unchanged even when the underlying physical topology changes. In linear regression with the synthetic dataset and $50$ workers, as observed in Fig.~\ref{dgadmm_eval}, even though D-GADMM consumes two iterations per topology change in building the chain, both the total number of iterations to achieve the objective error of $1E-4$ and the TC of D-GADMM are significantly reduced compared to GADMM. We observe that by changing the neighboring set of each worker more frequently, the convergence speed is significantly improved. Therefore, even for the static scenario in which the physical topology does not change, reconstructing the logical chain every few iterations can significantly improve the convergence speed. 

\begin{figure}
\centering
\includegraphics[trim=0.1in 0in 0.1in 0.1in, clip,  width=\textwidth]{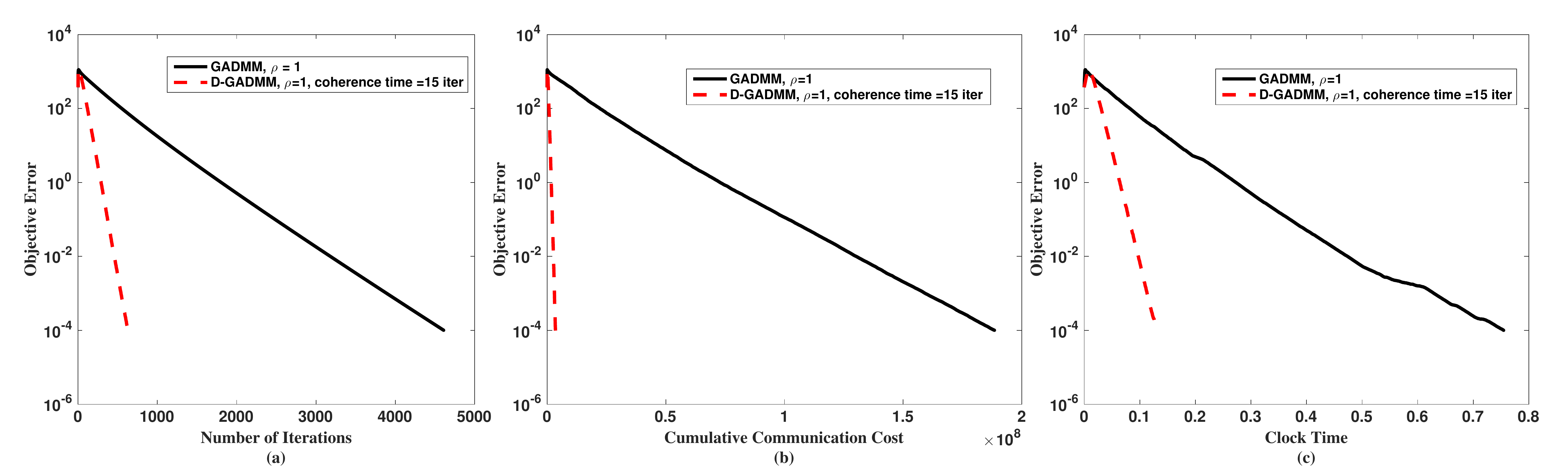}
\caption{ Objective error, total communication cost, and total running time of D-GADMM versus GADMM in linear regression with the synthetic dataset at $\rho=1$, $N=50$}
\label{dgadmm_eval}
\end{figure}

\begin{figure}
\centering
\includegraphics[trim=0.1in 0in 0.1in 0.1in, clip,  width=\textwidth]{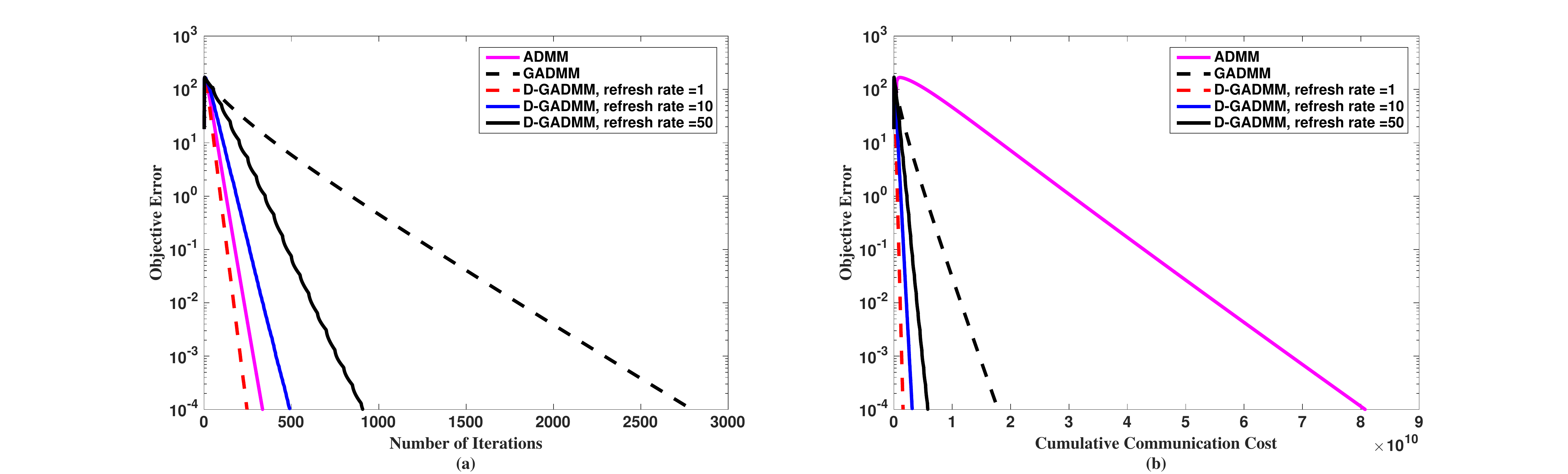}
\caption{ Objective error, total communication cost, and total running time of D-GADMM, GADMM, and Standared ADMM in linear regression with the synthetic dataset at $\rho=1$, $N=24$}
\label{gadmm_vs_admm}
\end{figure}

We finally compare both GADMM and D-GADMM with the standard ADMM which requires a parameter server (star topology). Since the topology does not change, we replace ``system coherence time'' with ``refresh rate''. Therefore, the objective of using  D-GADMM is not to adapt to topology changes, while to improve the convergence speed of GADMM. To compare the algorithms, we use $24$ workers ($N=24$), and we randomly drop them over a $250\!\times\!250$ m$^2$ square area. For standard ADMM, we use the worker that is closest to the center of the grid as the parameter server.


As observed from Fig.~\ref{gadmm_vs_admm}, compared to GADMM, standard ADMM requires fewer iterations to achieve the objective error of $1E-4$, but that comes at significantly higher communication cost as shown in~Fig.~\ref{gadmm_vs_admm}-(b) ($4$ times higher cost than GADMM). We show that by randomly changing the logical connectivity graph and utilizing D-GADMM, we can reduce the gap in the number of iterations between GADMM and standard ADMM and significantly reduce the communication cost. In fact, Fig.~\ref{gadmm_vs_admm} shows that by changing the logical graph every iteration, D-GADMM converges faster than standard ADMM and achieves a communication cost that is $40$ times less. It is worth mentioning that for static physical topology, changing the logical graph comes at zero cost since workers can agree on a predefined pseudorandom sequence in the graph changes. Therefore, every worker knows its neighbors in the next iteration.

\section{Conclusions and Future work}
\label{conclusion}
In this paper, we formulate a constrained optimization problem for distributed machine learning applications, and propose a novel decentralized algorithm based on ADMM, termed Group ADMM (GADMM) to solve this problem optimally for convex functions. GADMM is shown to maximize the communication efficiency of each worker. Extensive simulations in linear and logistic regression with synthetic and real datasets show significant improvements in convergence rate and communication overhead compared to the state-of-the-art algorithms. Furthermore, we extend GADMM to D-GADMM which accounts for time-varying network topologies. Both analysis and simulations confirm that D-GADMM achieves the same convergence guarantees as GADMM with lower communication overhead under the time-varying topology scenario. Constructing a communication-efficient logical chain may not always be possible; therefore, extending the algorithm to achieve a low communication overhead under an arbitrary topology could be an interesting topic for future study.

\appendix
\clearpage\newpage
\onecolumn
\section{Proof of Lemma \ref{lemma:first}}
\label{sec:lem1}
\emph{Proof of statement (i):}
We note that $f_n(\boldsymbol{\theta}_n)$ for all $n$ is closed, proper, and convex, hence $\boldsymbol{\mathcal{L}}_{\rho}$ is sub-differentiable. Since $\boldsymbol{\theta}_{n \in {\cal N}_h}^{k+1}$ minimizes $\boldsymbol{\mathcal{L}}_{\rho}( \boldsymbol{\theta}_{n \in {\cal N}_h},\boldsymbol{\theta}_{n\in {\cal N}_t}^{k}, \lamb_{n})$, the following must hold true at each iteration $k+1$
\begin{align}\label{first}
\boldsymbol{0} \in \partial f_n(\boldsymbol{\theta}_{n}^{k+1}) - \lamb_{n-1}^{k+1}+ \lamb_{n}^{k+1} +\bbs_{n}^{k+1}, n\in {\cal N}_h\setminus \{1\}
\end{align} 
\begin{align}\label{firstEdge}
\boldsymbol{0} \in \partial f_n(\boldsymbol{\theta}_{n}^{k+1}) + \lamb_{n}^{k+1} +\bbs_{n}^{k+1}, n=1
\end{align} 
Note that we use \eqref{firstEdge} for worker $1$ since it does not have a left neighbor (\ie  $\lamb_{0}^{k+1}$ is not defined). However, for simplicity and to avoid writing separate equations for the edge workers (workers $1$ and $N$), we use: $\lamb_{0}^{k+1}=\lamb_{N}^{k+1}=0$ throughout the rest of the proof. Therefore, we can use a single equation for each group (\eg equation \eqref{first} for $n\in {\cal N}_h$).

The result in \eqref{first} implies that $\boldsymbol{\theta}_{n}^{k+1}$  for $n\in {\cal N}_h$ minimizes the following convex objective function 
\begin{align}\label{second}
f_n(\boldsymbol{\theta}_{n}) + \ip{- \lamb_{n-1}^{k+1}+ \lamb_{n}^{k+1} + \bbs_{n \in {\cal N}_h}^{k+1},\boldsymbol{\theta}_n}.
\end{align}
Next, since $\boldsymbol{\theta}_{n}^{k+1}$ for $n\in {\cal N}_h$ is the minimizer of \eqref{second}, then, it holds that 
\begin{align}\label{third}
\!f_n(\boldsymbol{\theta}_{n}^{k+1}) + \ip{- \lamb_{n-1}^{k+1}+ \lamb_{n}^{k+1} + \bbs_{n\in {\cal N}_h}^{k+1},\boldsymbol{\theta}_{n}^{k+1}}\leq f_n(\boldsymbol{\theta}^\star)\! +\ip{- \lamb_{n-1}^{k+1}\!+\! \lamb_{n}^{k+1} \!+\! \bbs_{n \in {\cal N}_h}^{k+1},\boldsymbol{\theta}^\star} 
\end{align}  
where $\boldsymbol{\theta}^\star$ is the optimal value of the problem in \eqref{com_admm}-\eqref{com_admm_c1}. 
Similarly for $\boldsymbol{\theta}_{n}^{k+1}$ for $n\in {\cal N}_t$ satisfies \eqref{dual_feasibility} and it holds that 
\begin{align}\label{fourth}
& f_n(\boldsymbol{\theta}_{n}^{k+1}) + \ip{- \lamb_{n-1}^{k+1}+ \lamb_{n}^{k+1},\boldsymbol{\theta}_{n}^{k+1}} \leq f_n(\boldsymbol{\theta}^\star) + \ip{- \lamb_{n-1}^{k+1}+ \lamb_{n}^{k+1},\boldsymbol{\theta}^\star}.
\end{align}
Adding \eqref{third} and \eqref{fourth}, and then taking the summation over all the workers, we get
\begin{align}
\sum_{n=1}^N& f_n(\boldsymbol{\theta}_{n}^{k+1}) + \sum_{n\in {\cal N}_t}\ip{- \lamb_{n-1}^{k+1}+ \lamb_{n}^{k+1},\boldsymbol{\theta}_n^{k+1}} + \sum_{n\in {\cal N}_h}\ip{- \lamb_{n-1}^{k+1}+ \lamb_{n}^{k+1} + \bbs_{n \in {\cal N}_h}^{k+1},\boldsymbol{\theta}_n^{k+1}}\nonumber
\\
&\leq \sum_{n=1}^Nf_n(\boldsymbol{\theta}^\star)+ \sum_{n\in {\cal N}_t}\ip{- \lamb_{n-1}^{k+1}+ \lamb_{n}^{k+1},\boldsymbol{\theta}^\star}+ \sum_{n\in {\cal N}_h}\ip{- \lamb_{n-1}^{k+1}+ \lamb_{n}^{k+1} + \bbs_{n \in {\cal N}_h}^{k+1},\boldsymbol{\theta}^\star} 
\end{align}

After rearranging the terms, we get
\begin{align}\label{thrid}
\sum_{n=1}^Nf_n(\boldsymbol{\theta}_{n}^{k+1})-\sum_{n=1}^Nf_n(\boldsymbol{\theta}^\star)\leq&  \sum_{n\in {\cal N}_t}\ip{- \lamb_{n-1}^{k+1}+ \lamb_{n}^{k+1},\boldsymbol{\theta}^\star}+ \sum_{n\in {\cal N}_h}\ip{- \lamb_{n-1}^{k+1}+ \lamb_{n}^{k+1},\boldsymbol{\theta}^\star}\nonumber
\\
&  - \sum_{n\in {\cal N}_t}\ip{- \lamb_{n-1}^{k+1}+ \lamb_{n}^{k+1},\boldsymbol{\theta}_n^{k+1}} -\sum_{n\in {\cal N}_h}\ip{- \lamb_{n-1}^{k+1}+ \lamb_{n}^{k+1},\boldsymbol{\theta}_n^{k+1}}\nonumber
\\
&\quad+\sum_{n\in {\cal N}_h}\ip{\bbs_{n \in {\cal N}_h}^{k+1},\boldsymbol{\theta}^\star-\boldsymbol{\theta}_n^{k+1}}.
\end{align}
Note that,  
\begin{align}\label{fifth}
\sum_{n\in {\cal N}_h}\ip{- \lamb_{n-1}^{k+1}+ \lamb_{n}^{k+1},\boldsymbol{\theta}_n}=&\ip{\lamb_1^{k+1},{\bbtheta}_1}-\ip{\lamb_2^{k+1},\boldsymbol{\theta}_3}+\ip{\lamb_3^{k+1},\boldsymbol{\theta}_3}+\cdots
\nonumber
\\
&\cdots-\ip{\lamb_{N-2}^{k+1},\boldsymbol{\theta}_{N-1}}+\ip{\lamb_{N-1}^{k+1},\boldsymbol{\theta}_{N-1}},
\end{align}
and
\begin{align}\label{sixth}
\sum_{n\in {\cal N}_t}\ip{- \lamb_{n-1}^{k+1}+ \lamb_{n}^{k+1},\boldsymbol{\theta}_n}=&-\ip{\lamb_1^{k+1},\boldsymbol{\theta}_2}+\ip{\lamb_2^{k+1},\boldsymbol{\theta}_2}-\ip{\lamb_3^{k+1},\boldsymbol{\theta}_4}+\cdots\nonumber 
\\
&\cdots-\ip{\lamb_{N-1}^{k+1},\boldsymbol{\theta}_{N}}+\ip{\lamb_{N-1}^{k+1},\boldsymbol{\theta}_{N}}.
\end{align}
From \eqref{fifth} and \eqref{sixth}, at $\boldsymbol{\theta}_n^{k+1}$, it holds that
\begin{align}\label{proof_1}
\sum_{n\in {\cal N}_t}&\ip{- \lamb_{n-1}^{k+1}+ \lamb_{n}^{k+1},\boldsymbol{\theta}_n^{k+1}}+\sum_{n\in {\cal N}_h}\ip{- \lamb_{n-1}^{k+1}+ \lamb_{n}^{k+1},\boldsymbol{\theta}_n^{k+1}}\nonumber
\\
&= \ip{\lamb_1^{k+1},\boldsymbol{\theta}_1^{k+1}-\boldsymbol{\theta}_2^{k+1}}+\ip{\lamb_2^{k+1},\boldsymbol{\theta}_2^{k+1}-\boldsymbol{\theta}_3^{k+1}}+\cdots+\ip{\lamb_{N-1}^{k+1},\boldsymbol{\theta}_{N-1}^{k+1}-\boldsymbol{\theta}_N^{k+1}}\nonumber\\
&= \ip{\lamb_1^{k+1}, \bbr_{1,2}^{k+1}}+\ip{\lamb_2^{k+1},\bbr_{2,3}^{k+1}}+\cdots+\ip{\lamb_{N-1}^{k+1},\bbr_{N-1,N}^{k+1}},
\end{align}
where for the second equality, we have used the definition of primal residuals  defined after \eqref{feasiblity}. Similarly, it holds for $\boldsymbol{\theta}^\star$ that
\begin{align}
\hspace{0mm}\sum_{n\in {\cal N}_t}&\ip{- \lamb_{n-1}^{k+1}+ \lamb_{n}^{k+1},\boldsymbol{\theta}^\star}+\sum_{n\in {\cal N}_h}\ip{- \lamb_{n-1}^{k+1}+ \lamb_{n}^{k+1},\boldsymbol{\theta}^\star}\label{eq1a}
\\&= \ip{\lamb_1^{k+1},\boldsymbol{\theta}^\star}+\ip{\lamb_2^{k+1}-\lamb_1^{k+1},\boldsymbol{\theta}^\star}+\ip{\lamb_3^{k+1}-\lamb_2^{k+1},\boldsymbol{\theta}^\star}+\cdots+\ip{\lamb_{N}^{k+1}-\lamb_{N-1}^{k+1},\boldsymbol{\theta}^\star}\nonumber\\& = 0.\nonumber
\end{align}
The equality in \eqref{eq1a} holds since $\lamb_{N}^{k+1}=\boldsymbol{0}$. Next, substituting the results from \eqref{proof_1} and \eqref{eq1a} into \eqref{thrid}, we get
\begin{align}
&\sum_{n=1}^Nf_n(\boldsymbol{\theta}_{n}^{k+1})-\sum_{n=1}^Nf_n(\boldsymbol{\theta}^\star)\leq -\sum_{n=1}^{N-1}\ip{\lamb_n^{k+1},\bbr_{n,n+1}^{k+1}}+\sum_{n\in {\cal N}_h}\ip{\bbs_{n \in {\cal N}_h}^{k+1},\boldsymbol{\theta}^\star-\boldsymbol{\theta}_n^{k+1}},
\label{eq2a}
\end{align}
which concludes the proof of statement (i) of Lemma \ref{lemma:first}.

\emph{Proof of statement (ii):}
	The proof of this Lemma is along the similar line as in \cite[A.3]{boyd2011distributed} but is provided here for completeness.
	We note that for a saddle point  $(\boldsymbol{\theta}^\star,\{\lamb_{n}^\star\}_{\forall n})$ of   $\boldsymbol{\mathcal{L}}_{0}(\{\boldsymbol{\theta}_n\}_{\forall n}, \{\lamb_{n}\}_{\forall n})$, it holds that  
\begin{align}\label{temp}
	\boldsymbol{\mathcal{L}}_{0}( \boldsymbol{\theta}^\star, \{\lamb_{n}^\star\}_{\forall n} ) \leq \boldsymbol{\mathcal{L}}_{0}( \boldsymbol{\{\theta}_n^{k+1}\}\forall n, \{\lamb_{n}^\star\}\forall n )
	\end{align}
	for all $n$. Substituting the expression for the Lagrangian from \eqref{augmentedLag4} on the both sides of \eqref{temp}, we get
	\begin{align}	\sum_{n=1}^Nf_n(\boldsymbol{\theta}^\star)\leq \sum_{n=1}^Nf_n(\boldsymbol{\theta}_n^{k+1})+\sum_{n=1}^{N-1}\ip{\lamb_{n}^\star,\bbr_{n,n+1}^{k+1}}.
	\end{align}
	After rearranging the terms, we get 
	\begin{align}
	&\sum_{n=1}^Nf_n(\boldsymbol{\theta}_n^{k+1})-\sum_{n=1}^Nf_n(\boldsymbol{\theta}^\star)\geq -\sum_{n=1}^{N-1}\ip{\lamb_{n}^\star,\bbr_{n,n+1}^{k+1}}
	\label{eq3a}
	\end{align}
	which is the statement (ii) of Lemma \ref{lemma:first}.

\section{Proof of Theorem \ref{theorem}}
\label{sec:them1}
To proceed with the analysis, add \eqref{eq2a} and \eqref{eq3a}, multiply by $2$, we get 
\begin{align}
2\sum_{n=1}^{N-1}\ip{\lamb_n^{k+1}-\lamb_n^\star,\bbr_{n,n+1}^{k+1}}+2\sum_{n\in {\cal N}_h}\ip{\bbs_{n}^{k+1},\boldsymbol{\theta}_n^{k+1}-\boldsymbol{\theta}^\star}\leq 0.
\label{eq4a}
\end{align}
By applying, $\lamb_n^{k+1}=\lamb_n^{k}+\rho \bbr_{n,n+1}^{k+1}$ obtained from the dual update in \eqref{lambdaUpdateb}, \eqref{eq4a} can be recast as 

%
\begin{align}
&2\sum_{n=1}^{N-1}\ip{\lamb_n^{k}+\rho \bbr_{n,n+1}^{k+1}-\lamb_n^\star,\bbr_{n,n+1}^{k+1}}+2\sum_{n\in {\cal N}_h}\ip{\bbs_{n}^{k+1},\boldsymbol{\theta}_n^{k+1}-\boldsymbol{\theta}^\star}\leq 0.
\label{eq7a}
\end{align}
Note that the first term on the left hand side of \eqref{eq7a} can be written as
\begin{align}
\sum_{n=1}^{N-1}2\ip{\lamb_n^{k}-\lamb_n^\star,\bbr_{n,n+1}^{k+1}}+\rho  \norm{\bbr_{n,n+1}^{k+1}}^2 + \rho \norm{\bbr_{n,n+1}^{k+1}}^2.
\label{eq5a}
\end{align}
Replacing $\bbr_{n,n+1}^{k+1}$ in the first and second terms of \eqref{eq5a} with $\frac{\lamb_n^{k+1}-\lamb_n^k}{\rho}$, we get
\begin{align}
\sum_{n=1}^{N-1}({2}/{\rho})\ip{\lamb_n^{k}-\lamb_n^\star,\lamb_n^{k+1}-\lamb_n^k}+(1/\rho) \norm{\lamb_n^{k+1}-\lamb_n^k}^2 + \rho\norm{\bbr_{n,n+1}^{k+1}}^2.
\label{eq6a}
\end{align}
Using the equality $\lamb_n^{k+1}-\lamb_n^k= (\lamb_n^{k+1}-\lamb_n^\star)-(\lamb_n^{k}-\lamb_n^\star)$, we can rewrite \eqref{eq6a} as
\begin{align}
\sum_{n=1}^{N-1}&\!(2/\rho)\ip{\lamb_n^{k}\!-\!\lamb_n^\star,(\lamb_n^{k+1}\!\!-\!\lamb_n^\star)-(\lamb_n^{k}\!-\!\lamb_n^\star)}\!+\!(1/\rho) \norm{(\lamb_n^{k+1}\!-\!\lamb_n^\star)-(\lamb_n^{k}\!-\!\lamb_n^\star)}^2+ \rho\norm{\bbr_{n,n+1}^{k+1}}^2\nonumber
\\
=&\sum_{n=1}^{N-1}(2/\rho)\ip{\lamb_n^{k}-\lamb_n^\star,\lamb_n^{k+1}-\lamb_n^\star}-(2/\rho)\norm{\lamb_n^{k}-\lamb_n^\star}^2 + (1/\rho)\norm{ \lamb_n^{k+1}-\lamb_n^\star}^2\nonumber
\\
&\quad\quad-(2/\rho)\ip{\lamb_n^{k+1}-\lamb_n^\star,\lamb_n^{k}-\lamb_n^\star}+1/\rho\norm{\lamb_n^{k}-\lamb_n^\star}^2+\rho \norm{\bbr_{n,n+1}^{k+1}}^2
\\
=&\sum_{n=1}^{N-1}\Big[(1/\rho)\norm{\lamb_n^{k+1}-\lamb_n^\star}^2- (1/\rho)\norm{\lamb_n^{k}-\lamb_n^\star}^2+\rho\norm{\bbr_{n,n+1}^{k+1}}^2\Big]\label{term_1_ap}.
\end{align} 
Next, consider the second term on the left hand side of \eqref{eq7a}. From the equality \eqref{dualResidualEq}, it holds that 
\begin{align}\label{eq8a}
2\sum_{n\in {\cal N}_h}&\ip{\bbs_{n}^{k+1},\boldsymbol{\theta}_n^{k+1}-\boldsymbol{\theta}^\star}
\\
&=\sum_{n\in {\cal N}_h\setminus\{1\}}\Big(2\rho\ip{\boldsymbol{\theta}_{n-1}^{k+1}-\boldsymbol{\theta}_{n-1}^{k},\boldsymbol{\theta}_{n}^{k+1}-\boldsymbol{\theta}^\star}\Big)+\sum_{n\in {\cal N}_h}\Big(2\rho\ip{\boldsymbol{\theta}_{n+1}^{k+1}-\boldsymbol{\theta}_{n+1}^{k},\boldsymbol{\theta}_{n}^{k+1}-\boldsymbol{\theta}^\star}\Big).\nonumber
\end{align}
Note that $\boldsymbol{\theta}_n^{k+1}=-\bbr_{n-1,n}^{k+1}+\boldsymbol{\theta}_{n-1}^{k+1}=\bbr_{n,n+1}^{k+1}+\boldsymbol{\theta}_{n+1}^{k+1}, \forall n=\{2,\cdots,N-1\}$ , which implies that we can rewrite \eqref{eq8a} as follows
\begin{align}\label{42}
2\sum_{n\in {\cal N}_h}&\ip{\bbs_{n}^{k+1},\boldsymbol{\theta}_n^{k+1}-\boldsymbol{\theta}^\star}\nonumber
\\
=&\sum_{n\in {\cal N}_h\setminus \{1\}}\Big(-2\rho\ip{\boldsymbol{\theta}_{n-1}^{k+1}-\boldsymbol{\theta}_{n-1}^{k},\bbr_{n-1,n}^{k+1}}+2\rho\ip{\boldsymbol{\theta}_{n-1}^{k+1}-\boldsymbol{\theta}_{n-1}^{k},\boldsymbol{\theta}_{n-1}^{k+1}-\boldsymbol{\theta}^\star}\Big)\nonumber
\\
&\quad\quad\quad\quad+\sum_{n\in {\cal N}_h}\Big(2\rho\ip{\boldsymbol{\theta}_{n+1}^{k+1}-\boldsymbol{\theta}_{n+1}^{k},\bbr_{n,n+1}^{k+1}}+2\rho\ip{\boldsymbol{\theta}_{n+1}^{k+1}-\boldsymbol{\theta}_{n+1}^{k},\boldsymbol{\theta}_{n+1}^{k+1}-\boldsymbol{\theta}^\star}\Big).
\end{align}
Using the equalities,  
\begin{align}
\boldsymbol{\theta}_{n-1}^{k+1}-\boldsymbol{\theta}^\star =& (\boldsymbol{\theta}_{n-1}^{k+1}-\boldsymbol{\theta}_{n-1}^k)+(\boldsymbol{\theta}_{n-1}^{k}-\boldsymbol{\theta}^\star), \forall n\in {\cal N}_h \setminus\{1\}\nonumber
\\
\boldsymbol{\theta}_{n+1}^{k+1}-\boldsymbol{\theta}^\star =& (\boldsymbol{\theta}_{n+1}^{k+1}-\boldsymbol{\theta}_{n+1}^k)+(\boldsymbol{\theta}_{n+1}^{k}-\boldsymbol{\theta}^\star), \forall n\in {\cal N}_h
\end{align} we rewrite the right hand side of \eqref{42} as 
\begin{align} \label{new}
&\sum_{n\in {\cal N}_h \setminus\{1\}}\Big(- 2\rho\ip{\boldsymbol{\theta}_{n-1}^{k+1}-\boldsymbol{\theta}_{n-1}^{k},\bbr_{n-1,n}^{k+1}}+2\rho\ip{\boldsymbol{\theta}_{n-1}^{k+1}-\boldsymbol{\theta}_{n-1}^{k},(\boldsymbol{\theta}_{n-1}^{k+1}-\boldsymbol{\theta}_{n-1}^k)+(\boldsymbol{\theta}_{n-1}^{k}-\boldsymbol{\theta}^\star)}\Big)\nonumber
\\
&\quad\quad\quad+\sum_{n\in {\cal N}_h}\Big(2\rho\ip{\boldsymbol{\theta}_{n+1}^{k+1}-\boldsymbol{\theta}_{n+1}^{k},\bbr_{n,n+1}^{k+1}}
+2\rho\ip{\boldsymbol{\theta}_{n+1}^{k+1}-\boldsymbol{\theta}_{n+1}^{k},(\boldsymbol{\theta}_{n}^{k+1}-\boldsymbol{\theta}_{n+1}^k)+(\boldsymbol{\theta}_{n+1}^{k}-\boldsymbol{\theta}^\star)}\Big) \nonumber
\\
=&\sum_{n\in {\cal N}_h\setminus\{1\}}\Big(-2\rho\ip{\boldsymbol{\theta}_{n-1}^{k+1}-\boldsymbol{\theta}_{n-1}^{k},\bbr_{n-1,n}^{k+1}}+2\rho\norm{\boldsymbol{\theta}_{n-1}^{k+1} -  \boldsymbol{\theta}_{n-1}^k}^2+2\rho\ip{\boldsymbol{\theta}_{n-1}^{k+1}-\boldsymbol{\theta}_{n-1}^{k},\boldsymbol{\theta}_{n-1}^{k}-\boldsymbol{\theta}^\star}\Big) \nonumber
\\
&\quad\quad\quad+\sum_{n\in {\cal N}_h}\Big(2\rho\ip{\boldsymbol{\theta}_{n+1}^{k+1}-\boldsymbol{\theta}_{n+1}^{k},\bbr_{n,n+1}^{k+1}}+2\rho\norm{\boldsymbol{\theta}_{n+1}^{k+1} -  \boldsymbol{\theta}_{n+1}^k}^2+2\rho\ip{\boldsymbol{\theta}_{n+1}^{k+1}-\boldsymbol{\theta}_{n+1}^{k},\boldsymbol{\theta}_{n+1}^{k}-\boldsymbol{\theta}_{n+1}^\star}\Big).
\end{align}
Further using the equalities
\begin{align}
\boldsymbol{\theta}_{n-1}^{k+1}-\boldsymbol{\theta}_{n-1}^k =& (\boldsymbol{\theta}_{n-1}^{k+1}-\boldsymbol{\theta}^\star)-(\boldsymbol{\theta}_{n-1}^{k}-\boldsymbol{\theta}^\star), \forall n\in {\cal N}_h \setminus\{1\}
\nonumber 
\\
\boldsymbol{\theta}_{n+1}^{k+1}-\boldsymbol{\theta}_{n+1}^k =& (\boldsymbol{\theta}_{n+1}^{k+1}-\boldsymbol{\theta}^\star)-(\boldsymbol{\theta}_{n+1}^{k}-\boldsymbol{\theta}^\star), \forall n\in {\cal N}_h
\end{align}
we can write \eqref{new} as 
\begin{align}
&\!\!\sum_{n\in {\cal N}_h\setminus \{1\}}\!\!\!\Big(\!-2\rho\ip{\boldsymbol{\theta}_{n-1}^{k+1}-\boldsymbol{\theta}_{n-1}^{k},\bbr_{n-1,n}^{k+1}}\!+\!2\rho\norm{\boldsymbol{\theta}_{n-1}^{k+1} -  \boldsymbol{\theta}_{n-1}^k}^2\!
\!\!+\!\!2\rho\ip{(\boldsymbol{\theta}_{n-1}^{k+1}\!-\!\boldsymbol{\theta}^\star)\!-\!(\boldsymbol{\theta}_{n-1}^{k}\!-\!\boldsymbol{\theta}^\star),\boldsymbol{\theta}_{n-1}^{k}\!-\!\boldsymbol{\theta}^\star}\Big)\nonumber
\\
&\quad\quad+\sum_{n\in {\cal N}_h}\Big(2\rho\ip{\boldsymbol{\theta}_{n+1}^{k+1}-\boldsymbol{\theta}_{n+1}^{k},\bbr_{n,n+1}^{k+1}}+2\rho\norm{\boldsymbol{\theta}_{n+1}^{k+1} -  \boldsymbol{\theta}_{n+1}^k}^2\nonumber
\\
&\quad\quad\quad+2\rho\ip{(\boldsymbol{\theta}_{n+1}^{k+1}-\boldsymbol{\theta}^\star)-(\boldsymbol{\theta}_{n+1}^{k}-\boldsymbol{\theta}^\star),\boldsymbol{\theta}_{n+1}^{k}-\boldsymbol{\theta}^\star}\Big) 
\\
=&\sum_{n\in {\cal N}_h\setminus \{1\}}\!\!\!\Big(\!-2\rho\ip{\boldsymbol{\theta}_{n-1}^{k+1}-\boldsymbol{\theta}_{n-1}^{k},\bbr_{n-1,n}^{k+1}}\!+\!2\rho\norm{\boldsymbol{\theta}_{n-1}^{k+1} -  \boldsymbol{\theta}_{n-1}^k}^2\!
\!+2\rho\ip{\boldsymbol{\theta}_{n-1}^{k+1}\!-\!\boldsymbol{\theta}^\star,\boldsymbol{\theta}_{n-1}^{k}\!-\!\boldsymbol{\theta}^\star}\nonumber
\\
&\quad\quad-2\rho\norm{\boldsymbol{\theta}_{n-1}^{k}\!-\!\boldsymbol{\theta}^\star}^2\Big)+\sum_{n\in {\cal N}_h}\Big(2\rho\ip{\boldsymbol{\theta}_{n+1}^{k+1}-\boldsymbol{\theta}_{n+1}^{k},\bbr_{n,n+1}^{k+1}}+2\rho\norm{\boldsymbol{\theta}_{n+1}^{k+1} -  \boldsymbol{\theta}_{n+1}^k}^2\nonumber
\\
&\quad\quad\quad+2\rho\ip{\boldsymbol{\theta}_{n+1}^{k+1}-\boldsymbol{\theta}^\star,\boldsymbol{\theta}_{n+1}^{k}-\boldsymbol{\theta}^\star}-2\rho\norm{\boldsymbol{\theta}_{n+1}^{k} -  \boldsymbol{\theta}^\star}^2\Big).
\end{align} 
After rearranging the terms, we can write
\begin{align}
=&\sum_{n\in {\cal N}_h \setminus\{1\}}\Big(-2\rho\ip{\boldsymbol{\theta}_{n-1}^{k+1}-\boldsymbol{\theta}_{n-1}^{k},\bbr_{n-1,n}^{k+1}}+\rho \norm{\boldsymbol{\theta}_{n-1}^{k+1} -  \boldsymbol{\theta}_{n-1}^k}^2+\rho\norm{(\boldsymbol{\theta}_{n-1}^{k+1} -  \boldsymbol{\theta}^\star)-(\boldsymbol{\theta}_{n-1}^{k} -  \boldsymbol{\theta}^\star)}^2\nonumber
\\
&\quad\quad\quad+2\rho\ip{\boldsymbol{\theta}_{n-1}^{k+1}-\boldsymbol{\theta}^\star,\boldsymbol{\theta}_{n-1}^{k}-\boldsymbol{\theta}^\star}-2\rho\parallel \boldsymbol{\theta}_{n-1}^{k} -  \boldsymbol{\theta}^\star\parallel_2^2\Big)+\sum_{n\in {\cal N}_h}\Big( 2\rho\ip{\boldsymbol{\theta}_{n+1}^{k+1}-\boldsymbol{\theta}_{n+1}^{k},\bbr_{n,n+1}^{k+1}}\nonumber
\\
&\quad\quad\quad\quad+\rho\norm{\boldsymbol{\theta}_{n+1}^{k+1} -  \boldsymbol{\theta}_{n+1}^k}^2+\rho\norm{(\boldsymbol{\theta}_{n+1}^{k+1} -  \boldsymbol{\theta}^\star)-(\boldsymbol{\theta}_{n+1}^{k} -  \boldsymbol{\theta}^\star)}^2\nonumber 
\\
&\quad\quad\quad\quad\quad+2\rho\ip{\boldsymbol{\theta}_{n+1}^{k+1}-\boldsymbol{\theta}^\star,\boldsymbol{\theta}_{n+1}^{k}-\boldsymbol{\theta}^\star}-2\rho\norm{\boldsymbol{\theta}_{n+1}^{k} -  \boldsymbol{\theta}^\star}^2\Big).\label{here}
\end{align}
Next, expanding the square terms in \eqref{here}, we get 
\begin{align} \label{term_2_ap}
2\sum_{n\in {\cal N}_h}& \ip{\bbs_{n}^{k+1},\boldsymbol{\theta}_n^{k+1}-\boldsymbol{\theta}^\star}
\nonumber\\
=&\sum_{n\in {\cal N}_h\setminus\{1\}}\Big(-2\rho\ip{\boldsymbol{\theta}_{n-1}^{k+1}-\boldsymbol{\theta}_{n-1}^{k},\bbr_{n-1,n}^{k+1}}+\rho\norm{\boldsymbol{\theta}_{n-1}^{k+1} -  \boldsymbol{\theta}_{n-1}^k}^2
\\
&+\rho\norm{\boldsymbol{\theta}_{n-1}^{k+1} \!-\!  \boldsymbol{\theta}^\star}^2\!\!-\!\rho\norm{ \boldsymbol{\theta}_{n-1}^{k} -  \boldsymbol{\theta}^\star}^2\Big)+\sum_{n\in {\cal N}_h} \Big(2\rho\ip{\boldsymbol{\theta}_{n+1}^{k+1}-\boldsymbol{\theta}_{n+1}^{k},\bbr_{n,n+1}^{k+1}}\nonumber
\\
&\quad+\rho\norm{ \boldsymbol{\theta}_{n+1}^{k+1} -  \boldsymbol{\theta}_{n+1}^k}^2+\rho\norm{\boldsymbol{\theta}_{n+1}^{k+1} -  \boldsymbol{\theta}^\star}^2-\rho\norm{\boldsymbol{\theta}_{n+1}^{k} -  \boldsymbol{\theta}^\star}^2\Big).\nonumber
\end{align} 
Substituting the equalities from \eqref{term_1_ap} and \eqref{term_2_ap} to the left hand side of \eqref{eq7a}, we obtain
\begin{align}
\sum_{n=1}^{N-1}&\Big[(1/\rho)\norm{\lamb_n^{k+1}-\lamb_n^\star}^2- (1/\rho)\norm{\lamb_n^{k}-\lamb_n^\star}^2+\rho\norm{\bbr_{n,n+1}^{k+1}}^2\Big]\nonumber
\\
&+\sum_{n\in {\cal N}_h \setminus\{1\}}\Big(-2\rho\ip{\boldsymbol{\theta}_{n-1}^{k+1}-\boldsymbol{\theta}_{n-1}^{k},\bbr_{n-1,n}^{k+1}}+\rho\norm{\boldsymbol{\theta}_{n-1}^{k+1} -  \boldsymbol{\theta}_{n-1}^k}^2
\\
&\quad+\rho\norm{\boldsymbol{\theta}_{n-1}^{k+1} \!-\!  \boldsymbol{\theta}^\star}^2\!\!-\!\rho\norm{ \boldsymbol{\theta}_{n-1}^{k} -  \boldsymbol{\theta}^\star}^2\Big)+\sum_{n\in {\cal N}_h}\Big(2\rho\ip{\boldsymbol{\theta}_{n+1}^{k+1}-\boldsymbol{\theta}_{n+1}^{k},\bbr_{n,n+1}^{k+1}}\nonumber
\\
&\quad\quad+\rho\norm{ \boldsymbol{\theta}_{n+1}^{k+1} -  \boldsymbol{\theta}_{n+1}^k}^2+\rho\norm{\boldsymbol{\theta}_{n+1}^{k+1} -  \boldsymbol{\theta}^\star}^2-\rho\norm{\boldsymbol{\theta}_{n+1}^{k} -  \boldsymbol{\theta}^\star}^2\Big) \leq 0.
\end{align}
Multiplying both the sides by $-1$, we get 
\begin{align}\label{new_2}
\sum_{n=1}^{N-1}&\Big[-(1/\rho)\norm{\lamb_n^{k+1}-\lamb_n^\star}^2+ (1/\rho)\norm{\lamb_n^{k}-\lamb_n^\star}^2-\rho\norm{\bbr_{n,n+1}^{k+1}}^2\Big]\nonumber
\\
&-\sum_{n\in {\cal N}_h\setminus\{1\}}\Big(-2\rho\ip{\boldsymbol{\theta}_{n-1}^{k+1}-\boldsymbol{\theta}_{n-1}^{k},\bbr_{n-1,n}^{k+1}}+\rho\norm{\boldsymbol{\theta}_{n-1}^{k+1} -  \boldsymbol{\theta}_{n-1}^k}^2
\\
&\quad+\rho\norm{\boldsymbol{\theta}_{n-1}^{k+1} \!-\!  \boldsymbol{\theta}^\star}^2\!\!-\!\rho\norm{ \boldsymbol{\theta}_{n-1}^{k} -  \boldsymbol{\theta}^\star}^2\Big)+\sum_{n\in {\cal N}_h}\Big(2\rho\ip{\boldsymbol{\theta}_{n+1}^{k+1}-\boldsymbol{\theta}_{n+1}^{k},\bbr_{n,n+1}^{k+1}}\nonumber
\\
&\quad\quad+\rho\norm{ \boldsymbol{\theta}_{n+1}^{k+1} -  \boldsymbol{\theta}_{n+1}^k}^2+\rho\norm{\boldsymbol{\theta}_{n+1}^{k+1} -  \boldsymbol{\theta}^\star}^2-\rho\norm{\boldsymbol{\theta}_{n+1}^{k} -  \boldsymbol{\theta}^\star}^2\Big)\geq 0,
\end{align}
After rearranging the terms in \eqref{new_2} and using the definition of the Lyapunov function in \eqref{lyapEq}, we get

\begin{align}
V_{k+1}\leq V_k-&
\sum_{n=1}^{N-1}\rho \norm{\bbr_{n,n+1}^{k+1}}^2- \Big[\sum_{n\in {\cal N}_h\setminus\{1\}} \rho\norm{ \boldsymbol{\theta}_{n-1}^{k+1} -  \boldsymbol{\theta}_{n-1}^k}^2+\sum_{n\in {\cal N}_h} \rho\norm{ \boldsymbol{\theta}_{n+1}^{k+1} -  \boldsymbol{\theta}_{n+1}^k}^2\Big]\nonumber
\\
 - &\Big[\sum_{n\in {\cal N}_h\setminus\{1\}}-2\rho\ip{\boldsymbol{\theta}_{n-1}^{k+1}-\boldsymbol{\theta}_{n-1}^{k},\bbr_{n-1,n}^{k+1}}+\sum_{n\in {\cal N}_h}2\rho\ip{\boldsymbol{\theta}_{n+1}^{k+1}-\boldsymbol{\theta}_{n+1}^{k},\bbr_{n,n+1}^{k+1}}\Big].
\label{mainIneq}
\end{align}
 In order to prove that $k+1$ is a one step towards the optimal solution or the Lyapunov function decreases monotonically at each iteration, we need to show that the sum of the inner product terms on the right hand side of the inequality is positive. In other words, we need to prove that the term $\sum_{n\in {\cal N}_h\setminus\{1\}}-2\rho\ip{\boldsymbol{\theta}_{n-1}^{k+1}-\boldsymbol{\theta}_{n-1}^{k},\bbr_{n-1,n}^{k+1}}+\sum_{n\in {\cal N}_h}2\rho\ip{\boldsymbol{\theta}_{n+1}^{k+1}-\boldsymbol{\theta}_{n+1}^{k},\bbr_{n,n+1}^{k+1}}$ is always positive. Note that this term can be written as. 
\begin{align}\label{term}
\sum_{n\in {\cal N}_h\setminus\{1\}}&-2\rho\ip{\boldsymbol{\theta}_{n-1}^{k+1}-\boldsymbol{\theta}_{n-1}^{k},\bbr_{n-1,n}^{k+1}}+\sum_{n\in {\cal N}_h}2\rho\ip{\boldsymbol{\theta}_{n+1}^{k+1}-\boldsymbol{\theta}_{n+1}^{k},\bbr_{n,n+1}^{k+1}}
\\
=& 2\rho \Big[\ip{\bbr_{1,2}^{k+1},\boldsymbol{\theta}_2^{k+1}-\boldsymbol{\theta}_2^{k}}- \ip{\bbr_{2,3}^{k+1},\boldsymbol{\theta}_2^{k+1}-\boldsymbol{\theta}_2^{k}}+\ip{\bbr_{3,4}^{k+1},\boldsymbol{\theta}_4^{k+1}-\boldsymbol{\theta}_4^{k}} - \ip{\bbr_{4,5}^{k+1},\boldsymbol{\theta}_4^{k+1}-\boldsymbol{\theta}_4^{k}}+\nonumber
\\
&\hspace{9cm}\cdots+\bbr_{N-1,N}^{k+1}(\boldsymbol{\theta}_N^{k+1}-\boldsymbol{\theta}_N^{k})\Big]\nonumber\\
=& 2\rho\ip{\bbr_{1,2}^{k+1}-\bbr_{2,3}^{k+1},\boldsymbol{\theta}_2^{k+1}-\boldsymbol{\theta}_2^{k}}+2\rho\ip{\bbr_{3,4}^{k+1}-\bbr_{4,5}^{k+1},\boldsymbol{\theta}_4^{k+1}-\boldsymbol{\theta}_4^{k}}+\nonumber
\\
&\hspace{9cm}\cdots+2\rho\ip{\bbr_{N-1,N}^{k+1},\boldsymbol{\theta}_N^{k+1}-\boldsymbol{\theta}_N^{k}}\nonumber.
\end{align}
We know that $\boldsymbol{\theta}_{n\in {\cal N}_t}^{k+1}$ minimizes $f_n(\boldsymbol{\theta}_n) + \ip{-\lamb_{n-1}^{k+1} + \lamb_n^{k+1},\boldsymbol{\theta}_n}$; hence it holds that 
\begin{align}
& f_n(\boldsymbol{\theta}_n^{k+1}) + \ip{-\lamb_{n-1}^{k+1} + \lamb_n^{k+1},\boldsymbol{\theta}_n^{k+1}}\leq f_n(\boldsymbol{\theta}_n^k) + \ip{-\lamb_{n-1}^{k+1} + \lamb_n^{k+1},\boldsymbol{\theta}_n^k},
\label{ineq1}
\end{align}
Similarly,  $\boldsymbol{\theta}_{n\in {\cal N}_t}^{k}$ minimizes $f_n(\boldsymbol{\theta}_n) + \ip{-\lamb_{n-1}^{k} + \lamb_n^{k},\boldsymbol{\theta}_n}$, which implies that 
\begin{align}
& f_n(\boldsymbol{\theta}_n^k) + \ip{-\lamb_{n-1}^{k} + \lamb_n^{k},\boldsymbol{\theta}_n^{k}}\leq f_n(\boldsymbol{\theta}_n^{k+1}) + \ip{-\lamb_{n-1}^{k} + \lamb_n^{k},\boldsymbol{\theta}_n^{k+1}}.
\label{ineq2}
\end{align}
Adding \eqref{ineq1} and \eqref{ineq2}, we get
\begin{align}
\ip{(-\lamb_{n-1}^{k+1} + \lamb_n^{k+1})-(-\lamb_{n-1}^{k} + \lamb_n^{k}),\boldsymbol{\theta}_n^{k+1}-\boldsymbol{\theta}_n^k} \leq 0.
\end{align}
Further after rearranging, we get 
\begin{align}
\ip{(\lamb_{n-1}^{k}-\lamb_{n-1}^{k+1}) + (\lamb_{n}^{k+1}-\lamb_{n}^{k}),\boldsymbol{\theta}_n^{k+1}-\boldsymbol{\theta}_n^k} \leq 0.
\label{ineq3}
\end{align}
By knowing that $\bbr_{n-1,n}^{k+1}=(1/\rho)(\lamb_{n-1}^{k+1} - \lamb_{n-1}^{k})$ and $\bbr_{n,n+1}^{k+1}=(1/\rho)(\lamb_{n}^{k+1} - \lamb_{n}^{k})$, \eqref{ineq3} can be written as
\begin{align}
-\rho\ip{\bbr_{n-1,n}^{k+1}-\bbr_{n,n+1}^{k+1},\boldsymbol{\theta}_n^{k+1}-\boldsymbol{\theta}_n^k} \leq 0, \forall n \in {\cal N}_t.
\label{ineq4}
\end{align}
The above inequality implies that
\begin{align}
\rho\ip{\bbr_{n-1,n}^{k+1}-\bbr_{n,n+1}^{k+1},\boldsymbol{\theta}_n^{k+1}-\boldsymbol{\theta}_n^k} \geq 0, \forall n \in {\cal N}_t.
\label{ineq40}
\end{align}
Note that since worker $N$ does not have a right neighbor, $\bbr_{N,N+1}^{k+1}=\lamb_{N}^{k+1} =\lamb_{N}^{k} =0$.

Next, for $\rho >0$. Using the inequality in \eqref{ineq40} into  \eqref{term}, we get 

\begin{align}\label{term2}
\sum_{n\in {\cal N}_h\setminus\{1\}}&-2\rho\ip{\boldsymbol{\theta}_{n-1}^{k+1}-\boldsymbol{\theta}_{n-1}^{k},\bbr_{n-1,n}^{k+1}}+\sum_{n\in {\cal N}_h}2\rho\ip{\boldsymbol{\theta}_{n+1}^{k+1}-\boldsymbol{\theta}_{n+1}^{k},\bbr_{n,n+1}^{k+1}}\geq 0. 
\end{align}

Next, we use the result in \eqref{term2} into \eqref{mainIneq} to get 
\begin{align}
V_{k+1}\leq V_k-&
\sum_{n=1}^{N-1}\rho \norm{\bbr_{n,n+1}^{k+1}}^2- \Big[\sum_{n\in {\cal N}_h\setminus\{1\}} \rho\norm{ \boldsymbol{\theta}_{n-1}^{k+1} -  \boldsymbol{\theta}_{n-1}^k}^2+\sum_{n\in {\cal N}_h}\rho\norm{ \boldsymbol{\theta}_{n+1}^{k+1} -  \boldsymbol{\theta}_{n+1}^k}^2\Big].
\label{mainIneq2}
\end{align}
The result in \eqref{mainIneq2} proves that $V_{k
+1}$ decreases with $k$. Now, since $V_k\geq 0$ and $V_k\leq V_0$, it holds that $\Bigg[
\sum_{n=1}^{N-1}\rho \norm{\bbr_{n,n+1}^{k+1}}^2+ \Big[\sum_{n\in {\cal N}_h\setminus\{1\}} \rho\norm{ \boldsymbol{\theta}_{n-1}^{k+1} -  \boldsymbol{\theta}_{n-1}^k}^2+\sum_{n\in {\cal N}_h}\rho\norm{ \boldsymbol{\theta}_{n+1}^{k+1} -  \boldsymbol{\theta}_{n+1}^k}^2\Big]\Bigg].$ is bounded. Taking the telescopic sum over $k$ in \eqref{mainIneq2} as limit $K\rightarrow\infty$, we get 
\begin{align}
\lim_{K\rightarrow\infty}\sum\limits_{k=0}^{K}\Bigg[
\sum_{n=1}^{N-1}\rho \norm{\bbr_{n,n+1}^{k+1}}^2+ \Big[\sum_{n\in {\cal N}_h\setminus\{1\}} \rho\norm{ \boldsymbol{\theta}_{n-1}^{k+1} -  \boldsymbol{\theta}_{n-1}^k}^2+\sum_{n\in {\cal N}_h}\rho\norm{ \boldsymbol{\theta}_{n+1}^{k+1} -  \boldsymbol{\theta}_{n+1}^k}^2\Big]\Bigg] \leq V_0.
\label{mainIneq3}
\end{align}
The result in \eqref{mainIneq3} implies that  the primal residual $\bbr_{n,n+1}^{k+1}\rightarrow \boldsymbol{0}$ as $k\rightarrow\infty$  for all $n \in \{1,\cdots, N-1\}$ completing the proof of statement (i) in Theorem \ref{theorem}.  Similarly, the norm differences $\norm{ \boldsymbol{\theta}_{n-1}^{k+1} -  \boldsymbol{\theta}_{n-1}^k}$ and $\norm{ \boldsymbol{\theta}_{n+1}^{k+1} -  \boldsymbol{\theta}_{n+1}^k}\rightarrow\boldsymbol{0}$ as $k\rightarrow\infty$ which implies that the dual residual $\bbs_{n}^{k}\rightarrow\boldsymbol{0}$ as $k\rightarrow \infty$ for all $n\in {\cal N}_h$ stated in the statement (ii) of Theorem \ref{theorem}.  In order to prove the statement (iii) of Theorem \ref{theorem}, consider the lower and the upper bounds on the objective function optimality gap given by 
\begin{align}
\sum_{n=1}^N[f_n(\boldsymbol{\theta}_{n}^{k+1})-f_n(\boldsymbol{\theta}^\star)]&\leq -\sum_{n=1}^{N-1}\ip{\lamb_{n}^{k+1},\bbr_{n,n+1}^{k+1}}+\sum_{n\in {\cal N}_h}\ip{\bbs_{n}^{k+1},\boldsymbol{\theta}^\star-\boldsymbol{\theta}_n^{k+1}}
\label{lem1Eq23}\\
\sum_{n=1}^N[f_n(\boldsymbol{\theta}_{n}^{k+1})-f_n(\boldsymbol{\theta}^\star)]&\geq -\sum_{n=1}^{N-1}\ip{\lamb_{n}^\star,\bbr_{n,n+1}^{k+1}}
	\label{eq3a3}.
\end{align}
Note that from the results in statement (i) and (ii) of Theorem \ref{theorem}, it holds that the right hand side of the upper bound in \eqref{lem1Eq23} converge to zero as $k\rightarrow\infty$ and also the right hand side of the lower bound in \eqref{eq3a3} converges to zero as $k\rightarrow\infty$. This implies that 
\begin{align}
\lim_{k\rightarrow\infty}\sum_{n=1}^N[f_n(\boldsymbol{\theta}_{n}^{k+1})-f_n(\boldsymbol{\theta}^\star)] =0
\end{align}
which is the statement (iii) of Theorem \ref{theorem}.

\section{Method for D-GADMM Chain Construction}
\label{chainConst}
To ensure that a chain in a given graph is found in a decentralized way, we use the following method.
\begin{itemize}
\item The $N$ workers ($N$ is assumed to be even) share a pseudorandom code that is used every $\tau$ seconds, where $\tau$ is the system coherence time, to generate a set ${\cal H}$ containing $(N/2-2)$ integer numbers, with each number chosen from the set $\{2,\cdots,N-1\}$. In other words, we assume that the topology changes every $\tau$ seconds. Note that the generated numbers at $i\cdot\tau$ and $(i+1)\cdot\tau$ time slots may differ. However, at the $i\cdot\tau$-th time slot, the same set of numbers is generated across workers with no communication.

\item Every worker with physical index $n \in \big\{{\cal H}\cup \{1\}\big\}$ is assigned to the head set. Note that the worker whose physical index $1$ is always assigned to the head set. On the other hand, every worker with physical index $n$ such that $n \notin \big\{{\cal H}\cup \{1\}\big\}$ assigns itself to the tail set. Therefore, the worker whose physical index $N$ is always assigned to the tail set. Following this strategy, the number of heads will be equal to the number of tails, and are both equal to $N/2$. 

\item Every worker in the head set broadcasts its physical index alongside a pilot signal. Pilot signal is a signal known to each worker. It is used to measure the signal strength and find neighboring workers.

\item Every worker in the tail set calculates its cost of communication to every head based on the received signal strength. For example, the communication cost between head $n$ and tail $m$ is equal to [1/power of the received signal from head $n$ to tail $m$], in which the link with lower received signal level is more costly, as it is incuring higher transmission power. 

\item Every worker in the tail set broadcasts a vector of length $N/2$, containing the communication cost to the $N/2$ heads, \ie the first element in the vector captures the communication cost between this tail and worker $1$, since worker $1$ is always in the head set, whereas the second element represents the communication cost between this tail and the first index in the head set ${\cal H}$ and so on. 

\item Once head worker $n \in \big\{{\cal H}\cup \{1\}\big\}$ receives the communication cost vector from tail workers, it finds a communication-efficient chain that starts from worker $1$ and passes through every other worker to reach worker $N$. In our simulations, we use the following simple and greedy strategy that is performed by every head to ensure they generate the same chain. The strategy is as follows: 
\begin{itemize}

\item Find the tail that has the minimum communication cost to worker $1$ and create a link between this tail and worker $1$.
\item From the remaining set of heads, find the head that has the minimum communication cost to this tail and create a link between this head and the corresponding tail. 
\item Follow the same strategy until all workers are connected. When every head follows this strategy, all heads will generate the same chain. 
\item Under the following two assumptions: (i) the communication cost between any pair of workers is $<\infty$, and (ii) no two tails have equal communication cost to the same head, this strategy guarantees that every head will generate the same chain.
\end{itemize}

\item Once every head finds its chain, all neighbors share their current models, and D-GADMM is carried out for $\tau$ seconds using the current chain.
\end{itemize}
Note that, the described heuristic requires $4$ communication rounds ($2$ iterations). Finally, it is worth mentioning that  this approach has no guarantee to find the most communication-efficient chain. As mentioned in section \ref{dgadmm}, our focus in this paper is not to design the chain construction algorithm.

\section{Convergence Analysis of D-GADMM}\label{DGADMMproof}
For the dynamic settings, we assume that the first $n=1$ and the last node $n=N$ are fixed and the others can move at each iterate. Therefore, we denote the neighbors to each node $n$ at iteration $k$ as $n_{l,k}$ and $n_{r,k}$ as the left and the right neighbors , respectively. Note that when, $n_{l,k}=n-1$ and $n_{r,k}=n+1$ for all $k$, we recover the GADMM implementation.  With that in mind, we start by writing the augmented Lagrangian of the optimization problem in \eqref{com_admm}-\eqref{com_admm_c1} at each iteration $k$ as
\begin{align}
\boldsymbol{\mathcal{L}}_{\rho}^k(\{\boldsymbol{\theta}_n\}_{n=1}^N,\boldsymbol{\lambda})=&\sum_{n=1}^N f_n(\boldsymbol{\theta}_n) + \sum_{n=1}^{N-1} \ip{\lamb _n, \boldsymbol{\theta}_{n} - \boldsymbol{\theta}_{n_{r,k}}}+ \frac{\rho}{2}  \sum_{n=1}^{N-1} \| \boldsymbol{\theta}_{n} - \boldsymbol{\theta}_{n+1}\|^2, 
\label{augmentedLag41}
\end{align}
where $\boldsymbol{\lambda}:=[\lamb_1^T, \cdots, \lamb_{N-1}^T]^T$~is the collection of dual variables. Note that the set of nodes in head $\mathcal{N}_h^k$ and tail $\mathcal{N}_t^k$ will change with $k$.6 The primal and dual variables under GADMM are updated in the following three steps. 
The modified algorithm updates are written as 

\begin{enumerate}
\item At iteration $k+1$, the \emph{primal variables of head workers} are updated as:
\begin{align}
{\boldsymbol{\theta}}_{n}^{k+1} =\arg\min_{\bbtheta_n}\big[f_n(\boldsymbol{\theta}_n) +&\ip{\bblambda_{n_{l,k}}^{k}, \boldsymbol{\theta}_{n_{l,k}}^{k} - \boldsymbol{\theta}_{n}}+\ip{\lamb_{n}^{k}, \boldsymbol{\theta}_{n} - \boldsymbol{\theta}_{n_{r,k}}^{k}} + \frac{\rho}{2} \| \boldsymbol{\theta}_{n_{l,k}}^{k} - \boldsymbol{\theta}_{n}\|^2\nonumber
\\
+& \frac{\rho}{2} \| \boldsymbol{\theta}_{n} - \boldsymbol{\theta}_{n_{r,k}}^{k}\|^2\big], n \in {\cal N}_h\setminus\{1\}
\label{headUpdate1}
\end{align}
Since the first head worker ($n=1$) does not have a left neighbor ($\boldsymbol{\theta}_{n-1}$ is not defined), its model is updated as follows.
\begin{align}
{\boldsymbol{\theta}}_{n}^{k+1} =\arg\min_{\bbtheta_n}\big[f_n(\boldsymbol{\theta}_n)&+\ip{\lamb_{n}^{k}, \boldsymbol{\theta}_{n} - \boldsymbol{\theta}_{n_{r,k}}^{k}} 
+\frac{\rho}{2} \| \boldsymbol{\theta}_{n} - \boldsymbol{\theta}_{n_{l,k}}^{k}\|^2\big], n=1
\label{headUpdateEdge1}
\end{align}

\item After the updates in \eqref{headUpdate1} and \eqref{headUpdateEdge1}, head workers send their updates to their two tail neighbors. The \emph{primal variables of tail workers} are then updated as:
\begin{align}
{\boldsymbol{\theta}}_{n}^{k+1} =\arg\min_{\boldsymbol{\theta}_n} \big[f_n(\boldsymbol{\theta}_n) +& \ip{\bblambda_{n_{l,k}}^{k}, \boldsymbol{\theta}_{n_{l,k}}^{k+1}-\boldsymbol{\theta}_{n}}
+\ip{\lamb_{n}^{k}, \boldsymbol{\theta}_{n} - \boldsymbol{\theta}_{n_{r,k}}^{k+1}}+ \frac{\rho}{2} \| \boldsymbol{\theta}_{n_{l,k}}^{k+1} - \boldsymbol{\theta}_{n}\|^2\nonumber\\ +& \frac{\rho}{2} \| \boldsymbol{\theta}_{n} - \boldsymbol{\theta}_{n_{r,k}}^{k+1}\|^2\big], n \in {\cal N}_t\setminus\{N\}.
\label{tailUpdate1}
\end{align}
Since the last tail worker ($n=N$) does not have a right neighbor ($\boldsymbol{\theta}_{n+1}$ is not defined), its model is updated as follows
\begin{align}
{\boldsymbol{\theta}}_{n}^{k+1} =\arg\min_{\boldsymbol{\theta}_n} \big[f_n(\boldsymbol{\theta}_n) +& \ip{\bblambda_{n_{l,k}}^{k}, \boldsymbol{\theta}_{n_{l,k}}^{k+1}-\boldsymbol{\theta}_{n}}
+ \frac{\rho}{2} \| \boldsymbol{\theta}_{n_{l,k}}^{k+1} - \boldsymbol{\theta}_{n}\|^2\big], n=N.
\label{tailUpdateEdge1}
\end{align}

\item After receiving the updates from neighbors, \emph{every worker locally updates its dual variables} $\lamb_{n-1}$~and $\lamb_{n}$~as follows
\begin{align}
{\lamb}_n^{k+1}={\lamb}_n^{k} + \rho({\boldsymbol{\theta}}_{n}^{k+1} - {\boldsymbol{\theta}}_{n_{r,k}}^{k+1}), n=\{1,\cdots,N-1\}.
\label{lambdaUpdateb1}
\end{align}
\end{enumerate}
Note that when the topology changes, ${\lamb}_n^{k} $ of worker $n$ is received from the left neighbor $n_{l,k}$ before updating ${\lamb}_n^{k+1}$ according to \eqref{lambdaUpdateb1}.  For the proof, we start with the necessary and sufficient optimality conditions, which are the primal and the dual feasibility conditions \citep{boyd2011distributed} for each $k$ are defined as 
\begin{align}\label{primal_feasiblity1}
\boldsymbol{\theta}_n^\star =& \boldsymbol{\theta}_{n_{l,k}}^\star, n \in \{2,\cdots,N\} \quad \quad\quad\quad\quad\quad \quad \text{(primal feasibility)}
\end{align} 
\begin{align}\label{dual_feasibility1}
&\boldsymbol{0} \in \partial f_n(\boldsymbol{\theta}_n^\star) - \bblambda_{n_{l,k}}^\star + \lamb_{n}^\star, n \in \{2,\cdots,N-1\}\nonumber\\
&\boldsymbol{0} \in \partial f_n(\boldsymbol{\theta}_n^\star) + \lamb_{n}^\star, n=1
 \quad \quad\quad\quad\quad\quad \quad \quad \quad \text{(dual feasibility)}\nonumber\\
 &\boldsymbol{0} \in \partial f_n(\boldsymbol{\theta}_n^\star) + \lamb_{n-1}^\star, n=N
\quad \quad\quad\quad\quad\quad \quad \quad \quad
\end{align}
We remark that the  optimal values $\bbtheta_n^\star$ are equal for each $n$, we denote $\boldsymbol{\theta}^\star=\boldsymbol{\theta}_n^\star = \boldsymbol{\theta}_{n-1}^\star$ for all $n$. Note that, at iteration $k+1$, we calculate $\boldsymbol{\theta}_{n}^{k+1}$ for all $n\in {\cal N}_t^k\setminus \{N\}$ as in \eqref{tailUpdate}, from the first order optimality condition, it holds that  
\begin{align}
\boldsymbol{0} &\in \partial f_n(\boldsymbol{\theta}_{n}^{k+1}) - \bblambda_{n_{l,k}}^{k} + \lamb_{n}^{k} + \rho (\boldsymbol{\theta}_{n}^{k+1}- \boldsymbol{\theta}_{n_{l,k}}^{k+1})+ \rho (\boldsymbol{\theta}_{n}^{k+1} - \boldsymbol{\theta}_{n_{r,k}}^{k+1}).
\label{eq31}
\end{align}
Next, rewrite the equation in \eqref{eq31} as
\begin{align}
\boldsymbol{0} &\in \partial f_n(\boldsymbol{\theta}_{n}^{k+1}) - \big(\bblambda_{n_{l,k}}^{k} + \rho (\boldsymbol{\theta}_{n_{l,k}}^{k+1} - \boldsymbol{\theta}_{n}^{k+1})\big)+ \left(\lamb_{n}^{k} + \rho (\boldsymbol{\theta}_{n}^{k+1}- \boldsymbol{\theta}_{n_{r,k}}^{k+1})\right).
\label{eq41}
\end{align}
From the update in \eqref{lambdaUpdateb1}, the equation in
\eqref{eq41} implies that
\begin{align}
\boldsymbol{0} \in \partial f_n(\boldsymbol{\theta}_{n}^{k+1}) - \lamb_{n_{l,k}}^{k+1}+ \lamb_{n}^{k+1}, n\in {\cal N}_t^k\setminus \{N\}.
\label{slaveEq1}
\end{align}
Note that for the $N$-th worker, we calculate $\boldsymbol{\theta}_{N}^{k+1}$ as in \eqref{tailUpdateEdge}, then we follow the same steps, and we get 
\begin{align}
\boldsymbol{0} \in \partial f_N(\boldsymbol{\theta}_{N}^{k+1}) - \lamb_{N_{l,k}}^{k+1}.
\label{slaveEqEdge1}
\end{align}
From the result in \eqref{slaveEq1} and \eqref{slaveEqEdge1}, it holds that the dual feasibility condition in \eqref{dual_feasibility1} is always satisfied for all $n  \in {\cal N}_t
^k$. 

Next, consider every $\boldsymbol{\theta}_{n}^{k+1}$ such that $n\in {\cal N}_h^k\setminus\{1\}$ which is calculated as in \eqref{headUpdate1} at iteration $k$. Similarly from the first order optimality condition, we can write 
\begin{align}
\boldsymbol{0} &\in \partial f_n(\boldsymbol{\theta}_{n}^{k+1}) - \bblambda_{n_{l,k}}^{k} + \lamb_{n}^{k} + \rho (\boldsymbol{\theta}_{n}^{k+1} - \boldsymbol{\theta}_{n_{l,k}}^{k})+ \rho (\boldsymbol{\theta}_{n}^{k+1} - \boldsymbol{\theta}_{n_{r,k}}^{k})\label{feasi1}.
\end{align}
Note that in \eqref{feasi1}, we don't have all the primal variables calculated at $k+1$~instance. Hence, we add and subtract the terms $\boldsymbol{\theta}_{n_{l,k}}^{k+1}$ and $\boldsymbol{\theta}_{n_{r,k}}^{k+1}$ in \eqref{feasi1} to get 

\begin{align}
\boldsymbol{0} \in &\;\partial  f_n(\boldsymbol{\theta}_{n}^{k+1}) - \big(\bblambda_{n_{l,k}}^{k}+ \rho(\boldsymbol{\theta}_{n_{l,k}}^{k+1} - \boldsymbol{\theta}_{n}^{k+1})\big)+ \left(\lamb_{n}^{k} +\rho(\boldsymbol{\theta}_{n}^{k+1} - \boldsymbol{\theta}_{n_{r,k}}^{k+1})\right) \nonumber 
\\
&
+\rho(\boldsymbol{\theta}_{n_{l,k}}^{k+1}-\boldsymbol{\theta}_{n_{l,k}}^{k})+\rho(\boldsymbol{\theta}_{n_{r,k}}^{k+1}-\boldsymbol{\theta}_{n_{r,k}}^{k}).
\end{align}
From the update in \eqref{lambdaUpdateb1}, it holds that
\begin{align}\label{feasiblity1}
\boldsymbol{0} \in \partial & f_n(\boldsymbol{\theta}_{n}^{k+1}) - \lamb_{n_{l,k}}^{k+1}+ \lamb_{n}^{k+1}+\rho(\boldsymbol{\theta}_{n_{l,k}}^{k+1}-\boldsymbol{\theta}_{n_{l,k}}^{k})+\rho(\boldsymbol{\theta}_{n_{r,k}}^{k+1}-\boldsymbol{\theta}_{n_{r,k}}^{k}).
\end{align}
%
Following the same steps for the first head worker ($n=1$) after excluding the terms $\bblambda_{0}^{k}$ and $\rho (\boldsymbol{\theta}_{1}^{k+1} - \boldsymbol{\theta}_{0}^{k})$ from $\eqref{feasi1}$ (worker $1$ does not have a left neighbor) gives
\begin{align}\label{feasiblityEdge1}
\boldsymbol{0} \in \partial & f_1(\boldsymbol{\theta}_{1}^{k+1}) + \lamb_{1}^{k+1}+\rho(\boldsymbol{\theta}_{1_{r,k}}^{k+1}-\boldsymbol{\theta}_{1_{r,k}}^{k}).
\end{align}
Let $\bbs_{n}^{k+1}$, the dual residual of worker $n\in {\cal N}_h^k$~at iteration $k+1$, be defined as follows
\begin{equation}\label{dualResidualEq1}
\bbs_{n}^{k+1}
=\left\{\begin{array}{l}
\rho(\boldsymbol{\theta}_{n_{l,k}}^{k+1}-\boldsymbol{\theta}_{n_{l,k}}^{k})+\rho(\boldsymbol{\theta}_{n_{r,k}}^{k+1}-\boldsymbol{\theta}_{n_{r,k}}^{k}), \text{ for } n\in {\cal N}_h^k\setminus\{1\}\\
\rho(\boldsymbol{\theta}_{n_{r,k}}^{k+1}-\boldsymbol{\theta}_{n_{r,k}}^{k}),  \text{ for }  n = 1.
\end{array}\right.
\end{equation}
Next, we discuss about the primal feasibility condition in \eqref{primal_feasiblity1} at iteration $k+1$.  Let $\bbr_{n,n_{r,k}}^{k+1}= \boldsymbol{\theta}_{n}^{k+1}-\boldsymbol{\theta}_{n_{r,k}}^{k+1}$ be the primal residual of each worker~$n \in \{1,\cdots,N-1\}$. To show the convergence of GADMM, we need to prove that the conditions in \eqref{primal_feasiblity1}-\eqref{dual_feasibility1} are satisfied for each worker $n$.  We have already shown that the dual feasibility condition in \eqref{dual_feasibility1} is always satisfied for the tail workers, and the dual residual of tail workers is always zero. Therefore, to prove the convergence and the optimality of GADMM, we need to show that the $\bbr_{n, n_{r,k}}^{k}$ for all $n=1, \cdots, N-1$ and $\bbs_{n \in {\cal N}_h^k}^{k}$ converge to zero, and $\sum_{n=1}^Nf_n(\boldsymbol{\theta}_n^{k})$~converges to $\sum_{n=1}^Nf_n(\boldsymbol{\theta}^\star)$ as $k\rightarrow\infty$. We proceed as follows to prove the same. 

We note that $f_n(\boldsymbol{\theta}_n)$ for all $n$ is closed, proper, and convex, hence $\boldsymbol{\mathcal{L}}_{\rho}^k$ is sub-differentiable. Since $\boldsymbol{\theta}_{n}^{k+1}$  for $n \in {\cal N}_h^k$ at $k$ minimizes $\boldsymbol{\mathcal{L}}_{\rho}^k( \boldsymbol{\theta}_{n \in {\cal N}_h},\boldsymbol{\theta}_{n\in {\cal N}_t}^{k}, \lamb_{n})$, the following must hold true at each iteration $k+1$, which implies that 
\begin{align}\label{first11}
\boldsymbol{0} \in \partial f_n(\boldsymbol{\theta}_{n}^{k+1}) - \lamb_{n_{l,k}}^{k+1}+ \lamb_{n}^{k+1} +\bbs_{n}^{k+1}, n\in {\cal N}_h^k\setminus \{1\}
\end{align} 
\begin{align}\label{firstEdge11}
\boldsymbol{0} \in \partial f_1(\boldsymbol{\theta}_{1}^{k+1}) + \lamb_{1}^{k+1} +\bbs_{1}^{k+1}, n=1
\end{align} 
Note that we use \eqref{firstEdge11} for worker $1$ since it does not have a left neighbor (\ie  $\lamb_{0}^{k+1}$ is not defined). However, for simplicity and to avoid writing separate equations for the edge workers (workers $1$ and $N$), we use: $\lamb_{0}^{k+1}=\lamb_{N}^{k+1}=0$ throughout the rest of the proof. Therefore, we can use a single equation for each group (\eg equation \eqref{first} for $n\in {\cal N}_h^k$).

The result in \eqref{first11} implies that $\boldsymbol{\theta}_{n}^{k+1}$  for $n\in {\cal N}_h^k$ minimizes the following convex objective function 
\begin{align}\label{second11}
f_n(\boldsymbol{\theta}_{n}) + \ip{- \lamb_{n_{l,k}}^{k+1}+ \lamb_{n}^{k+1} + \bbs_{n }^{k+1},\boldsymbol{\theta}_n}.
\end{align}
Next, since $\boldsymbol{\theta}_{n}^{k+1}$ for $n\in {\cal N}_h^k$ is the minimizer of \eqref{second11}, then, it holds that 
\begin{align}\label{third11}
\!f_n(\boldsymbol{\theta}_{n}^{k+1}) + \ip{- \lamb_{n_{l,k}}^{k+1}+ \lamb_{n}^{k+1} + \bbs_{n}^{k+1},\boldsymbol{\theta}_{n}^{k+1}}\leq f_n(\boldsymbol{\theta}^\star)\! +\ip{- \lamb_{n_{l,k}}^{k+1}\!+\! \lamb_{n}^{k+1} \!+\! \bbs_{n }^{k+1},\boldsymbol{\theta}^\star} 
\end{align}  
where $\boldsymbol{\theta}^\star$ is the optimal value of the problem in \eqref{com_admm}-\eqref{com_admm_c1}. 
Similarly for $\boldsymbol{\theta}_{n}^{k+1}$ for $n\in {\cal N}_t^k$ satisfies \eqref{dual_feasibility1} and it holds that 
\begin{align}\label{fourth11}
& f_n(\boldsymbol{\theta}_{n}^{k+1}) + \ip{- \lamb_{n_{l,k}}^{k+1}+ \lamb_{n}^{k+1},\boldsymbol{\theta}_{n}^{k+1}} \leq f_n(\boldsymbol{\theta}^\star) + \ip{- \lamb_{n_{l,k}}^{k+1}+ \lamb_{n}^{k+1},\boldsymbol{\theta}^\star}.
\end{align}
Add \eqref{third11} and \eqref{fourth11}, and then take the summation over all the workers, note that for a given $k$, the topology in the network is fixed, we get
\begin{align}
\sum_{n=1}^N& f_n(\boldsymbol{\theta}_{n}^{k+1}) + \sum_{n\in {\cal N}_t^k}\ip{- \lamb_{n_{l,k}}^{k+1}+ \lamb_{n}^{k+1},\boldsymbol{\theta}_n^{k+1}} + \sum_{n\in {\cal N}_h^k}\ip{- \lamb_{n_{l,k}}^{k+1}+ \lamb_{n}^{k+1} + \bbs_{n }^{k+1},\boldsymbol{\theta}_n^{k+1}}\nonumber
\\
&\leq \sum_{n=1}^Nf_n(\boldsymbol{\theta}^\star)+ \sum_{n\in {\cal N}_t^k}\ip{- \lamb_{n_{l,k}}^{k+1}+ \lamb_{n}^{k+1},\boldsymbol{\theta}^\star}+ \sum_{n\in {\cal N}_h^k}\ip{- \lamb_{n_{l,k}}^{k+1}+ \lamb_{n}^{k+1} + \bbs_{n }^{k+1},\boldsymbol{\theta}^\star} 
\end{align}

After rearranging the terms, we get
\begin{align}\label{thrid11}
\sum_{n=1}^Nf_n(\boldsymbol{\theta}_{n}^{k+1})-\sum_{n=1}^Nf_n(\boldsymbol{\theta}^\star)\leq&  \sum_{n\in {\cal N}_t^k}\ip{- \lamb_{n_{l,k}}^{k+1}+ \lamb_{n}^{k+1},\boldsymbol{\theta}^\star}+ \sum_{n\in {\cal N}_h^k}\ip{- \lamb_{n_{l,k}}^{k+1}+ \lamb_{n}^{k+1},\boldsymbol{\theta}^\star}\nonumber
\\
&  - \sum_{n\in {\cal N}_t^k}\ip{- \lamb_{n_{l,k}}^{k+1}+ \lamb_{n}^{k+1},\boldsymbol{\theta}_n^{k+1}} -\sum_{n\in {\cal N}_h^k}\ip{- \lamb_{n_{l,k}}^{k+1}+ \lamb_{n}^{k+1},\boldsymbol{\theta}_n^{k+1}}\nonumber
\\
&\quad+\sum_{n\in {\cal N}_h^k}\ip{\bbs_{n }^{k+1},\boldsymbol{\theta}^\star-\boldsymbol{\theta}_n^{k+1}}.
\end{align}
Note that we can write  

\begin{align}\label{proof_111}
\sum_{n\in {\cal N}_t^t}&\ip{- \lamb_{n_{l,k}}^{k+1}+ \lamb_{n}^{k+1},\boldsymbol{\theta}_n^{k+1}}+\sum_{n\in {\cal N}_h^k}\ip{- \lamb_{n_{l,k}}^{k+1}+ \lamb_{n}^{k+1},\boldsymbol{\theta}_n^{k+1}}= \sum\limits_{n=1}^{N-1}\ip{\lamb_n^{k+1}, \bbr_{n,n_r^k}^{k+1}},
\end{align}
where for the equality, we have used the definition of primal residuals  defined after \eqref{feasiblity1}. Similarly, it holds for $\boldsymbol{\theta}^\star$ as
\begin{align}
\hspace{0mm}\sum_{n\in {\cal N}_t^k}&\ip{- \lamb_{n_{l,k}}^{k+1}+ \lamb_{n}^{k+1},\boldsymbol{\theta}^\star}+\sum_{n\in {\cal N}_h^k}\ip{- \lamb_{n_{l,k}}^{k+1}+ \lamb_{n}^{k+1},\boldsymbol{\theta}^\star}= 0.\label{eq1a11}
\end{align}
The equality in \eqref{eq1a11} holds since $\lamb_{N}^{k+1}=\boldsymbol{0}$. Next, substituting the results from \eqref{proof_111} and \eqref{eq1a11} into \eqref{thrid11}, we get
\begin{align}
&\sum_{n=1}^Nf_n(\boldsymbol{\theta}_{n}^{k+1})-\sum_{n=1}^Nf_n(\boldsymbol{\theta}^\star)\leq -\sum_{n=1}^{N-1}\ip{\lamb_n^{k+1},\bbr_{n,n_{r,k}}^{k+1}}+\sum_{n\in {\cal N}_h^k}\ip{\bbs_{n }^{k+1},\boldsymbol{\theta}^\star-\boldsymbol{\theta}_n^{k+1}},
\label{eq2a11}
\end{align}
which provides an upper bound on the optimality gap. Next, we get the lower bound as follows.
	We note that for a saddle point  $(\boldsymbol{\theta}^\star,\{\lamb_{n}^\star\}_{\forall n})$ of   $\boldsymbol{\mathcal{L}}_{0}(\{\boldsymbol{\theta}_n\}_{\forall n}, \{\lamb_{n}\}_{\forall n})$, it holds that  
	\begin{align}\label{temp11}
	\boldsymbol{\mathcal{L}}_{0}( \boldsymbol{\theta}^\star, \{\lamb_{n}^\star\}_{\forall n} ) \leq \boldsymbol{\mathcal{L}}_{0}( \{\boldsymbol{\theta}_n^{k+1}\}_{\forall n}, \{\lamb_{n}^\star\}_{\forall n} ).
	\end{align}
	Substituting the expression for the Lagrangian from \eqref{augmentedLag41} on the both sides of \eqref{temp11}, we get
	\begin{align}	\sum_{n=1}^Nf_n(\boldsymbol{\theta}^\star)\leq \sum_{n=1}^Nf_n(\boldsymbol{\theta}_n^{k+1})+\sum_{n=1}^{N-1}\ip{\lamb_{n}^\star,\bbr_{n,n_{r,k}}^{k+1}}.
	\end{align}
	After rearranging the terms, we get 
	\begin{align}
	&\sum_{n=1}^Nf_n(\boldsymbol{\theta}_n^{k+1})-\sum_{n=1}^Nf_n(\boldsymbol{\theta}^\star)\geq -\sum_{n=1}^{N-1}\ip{\lamb_{n}^\star,\bbr_{n,n_{r,k}}^{k+1}}
	\label{eq3a11}
	\end{align}
	which provide the lower bound on the optimality gap. Next, we show that both the lower and upper bound converges to zero as $\rightarrow\infty$. This would  prove that the optimality gap converges to zero with $k\rightarrow\infty$. 

%
To proceed with the analysis, add \eqref{eq2a11} and \eqref{eq3a11}, multiply by $2$, we get 
\begin{align}
2\sum_{n=1}^{N-1}\ip{\lamb_n^{k+1}-\lamb_n^\star,\bbr_{n,n_{r,k}}^{k+1}}+2\sum_{n\in {\cal N}_h^k}\ip{\bbs_{n}^{k+1},\boldsymbol{\theta}_n^{k+1}-\boldsymbol{\theta}^\star}\leq 0.
\label{eq4a11}
\end{align}
From the dual update in \eqref{lambdaUpdateb1}, we have $\lamb_n^{k+1}=\lamb_n^{k}+\rho \bbr_{n,n_{r,k}}^{k+1}$ and \eqref{eq4a11} can be written as 
\begin{align}
&2\sum_{n=1}^{N-1}\ip{\lamb_n^{k}+\rho \bbr_{n,n_{r,k}}^{k+1}-\lamb_n^\star,\bbr_{n,n_{r,k}}^{k+1}}+2\sum_{n\in {\cal N}_h^k}\ip{\bbs_{n}^{k+1},\boldsymbol{\theta}_n^{k+1}-\boldsymbol{\theta}^\star}\leq 0.
\label{eq7a11}
\end{align}
Note that the first term on the left hand side of \eqref{eq7a11} can be written as
\begin{align}
\sum_{n=1}^{N-1}2\ip{\lamb_n^{k}-\lamb_n^\star,\bbr_{n,n_{r,k}}^{k+1}}+\rho  \norm{\bbr_{n,n_{r,k}}^{k+1}}^2 + \rho \norm{\bbr_{n,n_{r,k}}^{k+1}}^2.
\label{eq5a11}
\end{align}
Replacing $\bbr_{n,n_{r,k}}^{k+1}$ in the first and second terms of \eqref{eq5a11} with $\frac{\lamb_n^{k+1}-\lamb_n^k}{\rho}$, we get
\begin{align}
\sum_{n=1}^{N-1}({2}/{\rho})\ip{\lamb_n^{k}-\lamb_n^\star,\lamb_n^{k+1}-\lamb_n^k}+(1/\rho) \norm{\lamb_n^{k+1}-\lamb_n^k}^2 + \rho\norm{\bbr_{n,n_{r,k}}^{k+1}}^2.
\label{eq6a11}
\end{align}
Using the equality $\lamb_n^{k+1}-\lamb_n^k= (\lamb_n^{k+1}-\lamb_n^\star)-(\lamb_n^{k}-\lamb_n^\star)$, we can write \eqref{eq6a11} as
\begin{align}
\sum_{n=1}^{N-1}&\!(2/\rho)\ip{\lamb_n^{k}\!-\!\lamb_n^\star,(\lamb_n^{k+1}\!\!-\!\lamb_n^\star)-(\lamb_n^{k}\!-\!\lamb_n^\star)}\!+\!(1/\rho) \norm{(\lamb_n^{k+1}\!-\!\lamb_n^\star)-(\lamb_n^{k}\!-\!\lamb_n^\star)}^2+ \rho\norm{\bbr_{n,n_{r,k}}^{k+1}}^2\nonumber
\\
=&\sum_{n=1}^{N-1}(2/\rho)\ip{\lamb_n^{k}-\lamb_n^\star,\lamb_n^{k+1}-\lamb_n^\star}-(2/\rho)\norm{\lamb_n^{k}-\lamb_n^\star}^2 + (1/\rho)\norm{ \lamb_n^{k+1}-\lamb_n^\star}^2\nonumber
\\
&\quad\quad-(2/\rho)\ip{\lamb_n^{k+1}-\lamb_n^\star,\lamb_n^{k}-\lamb_n^\star}+1/\rho\norm{\lamb_n^{k}-\lamb_n^\star}^2+\rho \norm{\bbr_{n,n_{r,k}}^{k+1}}^2
\\
=&\sum_{n=1}^{N-1}\Big[(1/\rho)\norm{\lamb_n^{k+1}-\lamb_n^\star}^2- (1/\rho)\norm{\lamb_n^{k}-\lamb_n^\star}^2+\rho\norm{\bbr_{n,n_{r,k}}^{k+1}}^2\Big]\label{term_1_ap11}.
\end{align} 
Next, consider the second term on the left hand side of \eqref{eq7a11}, from the equality \eqref{dualResidualEq1}, it holds that 
\begin{align}\label{eq8a11}
2\sum_{n\in {\cal N}_h^k}&\ip{\bbs_{n}^{k+1},\boldsymbol{\theta}_n^{k+1}-\boldsymbol{\theta}^\star}
\\
&=\sum_{n\in {\cal N}_h\setminus\{1\}}\Big(2\rho\ip{\boldsymbol{\theta}_{n_{l,k}}^{k+1}-\boldsymbol{\theta}_{n_{l,k}}^{k},\boldsymbol{\theta}_{n}^{k+1}-\boldsymbol{\theta}^\star}\Big)+\sum_{n\in {\cal N}_h}\Big(2\rho\ip{\boldsymbol{\theta}_{n_{r,k}}^{k+1}-\boldsymbol{\theta}_{n_{r,k}}^{k},\boldsymbol{\theta}_{n}^{k+1}-\boldsymbol{\theta}^\star}\Big).\nonumber
\end{align}
Note that $\boldsymbol{\theta}_n^{k+1}=-\bbr_{n_{l,k},n}^{k+1}+\boldsymbol{\theta}_{n_{l,k}}^{k+1}=\bbr_{n,n_{r,k}}^{k+1}+\boldsymbol{\theta}_{n_{r,k}}^{k+1}, \forall n=\{2,\cdots,N-1\}$ , which implies that we can rewrite the equation in \eqref{eq8a11} as follows
\begin{align}\label{4211}
2\sum_{n\in {\cal N}_h^k}&\ip{\bbs_{n}^{k+1},\boldsymbol{\theta}_n^{k+1}-\boldsymbol{\theta}^\star}\nonumber
\\
=&\sum_{n\in {\cal N}_h^k\setminus \{1\}}\Big(-2\rho\ip{\boldsymbol{\theta}_{n_{l,k}}^{k+1}-\boldsymbol{\theta}_{n_{l,k}}^{k},\bbr_{n_{l,k},n}^{k+1}}+2\rho\ip{\boldsymbol{\theta}_{n_{l,k}}^{k+1}-\boldsymbol{\theta}_{n_{l,k}}^{k},\boldsymbol{\theta}_{n_{l,k}}^{k+1}-\boldsymbol{\theta}^\star}\Big)\nonumber
\\
&\quad\quad\quad\quad+\sum_{n\in {\cal N}_h}\Big(2\rho\ip{\boldsymbol{\theta}_{n_{r,k}}^{k+1}-\boldsymbol{\theta}_{n_{r,k}}^{k},\bbr_{n,n_{r,k}}^{k+1}}+2\rho\ip{\boldsymbol{\theta}_{n_{r,k}}^{k+1}-\boldsymbol{\theta}_{n_{r,k}}^{k},\boldsymbol{\theta}_{n_{r,k}}^{k+1}-\boldsymbol{\theta}^\star}\Big).
\end{align}
Using the equalities,  
\begin{align}
\boldsymbol{\theta}_{n_{l,k}}^{k+1}-\boldsymbol{\theta}^\star =& (\boldsymbol{\theta}_{n_{l,k}}^{k+1}-\boldsymbol{\theta}_{n_{l,k}}^k)+(\boldsymbol{\theta}_{n_{l,k}}^{k}-\boldsymbol{\theta}^\star), \forall n\in {\cal N}_h^k \setminus\{1\}\nonumber
\\
\boldsymbol{\theta}_{n_{r,k}}^{k+1}-\boldsymbol{\theta}^\star =& (\boldsymbol{\theta}_{n_{r,k}}^{k+1}-\boldsymbol{\theta}_{n_{r,k}}^k)+(\boldsymbol{\theta}_{n_{r,k}}^{k}-\boldsymbol{\theta}^\star), \forall n\in {\cal N}_h^k
\end{align} we rewrite the right hand side of \eqref{4211} as 
\begin{align} \label{new11}
&\!\!\!\!\sum_{n\in {\cal N}_h^k\setminus\{1\}}\Big(\!\!\!-2\rho\ip{\boldsymbol{\theta}_{n_{l,k}}^{k+1}-\boldsymbol{\theta}_{n_{l,k}}^{k},\bbr_{n_{l,k},n}^{k+1}}+2\rho\|\boldsymbol{\theta}_{n_{l,k}}^{k+1} -  \boldsymbol{\theta}_{n_{l,k}}^k\|^2+2\rho\ip{\boldsymbol{\theta}_{n_{l,k}}^{k+1}-\boldsymbol{\theta}_{n_{l,k}}^{k},\boldsymbol{\theta}_{n_{l,k}}^{k}-\boldsymbol{\theta}^\star}\Big) \nonumber
\\
&\quad+\sum_{n\in {\cal N}_h^k}\Big(2\rho\ip{\boldsymbol{\theta}_{n_{r,k}}^{k+1}-\boldsymbol{\theta}_{n_{r,k}}^{k},\bbr_{n,n_{r,k}}^{k+1}}+2\rho\|\boldsymbol{\theta}_{n_{r,k}}^{k+1} -  \boldsymbol{\theta}_{n_{r,k}}^k\|^2+2\rho\ip{\boldsymbol{\theta}_{n_{r,k}}^{k+1}-\boldsymbol{\theta}_{n_{r,k}}^{k},\boldsymbol{\theta}_{n_{r,k}}^{k}-\boldsymbol{\theta}_{n_{r,k}}^\star}\Big).
\end{align}
Further using the equalities
\begin{align}
\boldsymbol{\theta}_{n_{l,k}}^{k+1}-\boldsymbol{\theta}_{n_{l,k}}^k =& (\boldsymbol{\theta}_{n_{l,k}}^{k+1}-\boldsymbol{\theta}^\star)-(\boldsymbol{\theta}_{n_{l,k}}^{k}-\boldsymbol{\theta}^\star), \forall n\in {\cal N}_h^k \setminus\{1\}
\nonumber 
\\
\boldsymbol{\theta}_{n_{r,k}}^{k+1}-\boldsymbol{\theta}_{n_{r,k}}^k =& (\boldsymbol{\theta}_{n_{r,k}}^{k+1}-\boldsymbol{\theta}^\star)-(\boldsymbol{\theta}_{n_{r,k}}^{k}-\boldsymbol{\theta}^\star), \forall n\in {\cal N}_h^k,
\end{align}
we can write \eqref{new11} after the rearrangement as  
\begin{align}
&\sum_{n\in {\cal N}_h^k \setminus\{1\}}\Big(-2\rho\ip{\boldsymbol{\theta}_{n_{l,k}}^{k+1}-\boldsymbol{\theta}_{n_{l,k}}^{k},\bbr_{n-1,n}^{k+1}}+\rho \|\boldsymbol{\theta}_{n_{l,k}}^{k+1} -  \boldsymbol{\theta}_{n_{l,k}}^k\|^2+\rho\|(\boldsymbol{\theta}_{n_{l,k}}^{k+1} -  \boldsymbol{\theta}^\star)-(\boldsymbol{\theta}_{n_{l,k}}^{k} -  \boldsymbol{\theta}^\star)\|^2\nonumber
\\ 
&\quad\quad\quad+2\rho\ip{\boldsymbol{\theta}_{n_{l,k}}^{k+1}-\boldsymbol{\theta}^\star,\boldsymbol{\theta}_{n_{l,k}}^{k}-\boldsymbol{\theta}^\star}-2\rho\parallel \boldsymbol{\theta}_{n_{l,k}}^{k} -  \boldsymbol{\theta}^\star\parallel_2^2\Big)+\sum_{n\in {\cal N}_h^k}\Big( 2\rho\ip{\boldsymbol{\theta}_{n_{r,k}}^{k+1}-\boldsymbol{\theta}_{n_{r,k}}^{k},\bbr_{n,n_{r,k}}^{k+1}}\nonumber
\\
&\quad\quad\quad\quad+\rho\|\boldsymbol{\theta}_{n_{r,k}}^{k+1} -  \boldsymbol{\theta}_{n_{r,k}}^k\|^2+\rho\|(\boldsymbol{\theta}_{n_{r,k}}^{k+1} -  \boldsymbol{\theta}^\star)-(\boldsymbol{\theta}_{n_{r,k}}^{k} -  \boldsymbol{\theta}^\star)\|^2\nonumber 
\\
&\quad\quad\quad\quad\quad+2\rho\ip{\boldsymbol{\theta}_{n_{r,k}}^{k+1}-\boldsymbol{\theta}^\star,\boldsymbol{\theta}_{n_{r,k}}^{k}-\boldsymbol{\theta}^\star}-2\rho\|\boldsymbol{\theta}_{n_{r,k}}^{k} -  \boldsymbol{\theta}^\star|^2\Big).\label{here11}
\end{align}
Next, expanding the square terms in \eqref{here11}, we get the upper bound for the term in \eqref{term_1_ap11} as follows
\begin{align} \label{term_2_ap11}
&\sum_{n\in {\cal N}_h^k\setminus\{1\}}\Big(-2\rho\ip{\boldsymbol{\theta}_{n_{l,k}}^{k+1}-\boldsymbol{\theta}_{n_{l,k}}^{k},\bbr_{n-1,n}^{k+1}}+\rho\|\boldsymbol{\theta}_{n_{l,k}}^{k+1} -  \boldsymbol{\theta}_{n_{l,k}}^k\|^2+\rho\|\boldsymbol{\theta}_{n_{l,k}}^{k+1} \!-\!  \boldsymbol{\theta}^\star\|^2\!\!-\!\rho\| \boldsymbol{\theta}_{n_{l,k}}^{k} -  \boldsymbol{\theta}^\star\|^2\Big)\nonumber
\\
&+\sum_{n\in {\cal N}_h^k} \Big(2\rho\ip{\boldsymbol{\theta}_{n_{r,k}}^{k+1}-\boldsymbol{\theta}_{n_{r,k}}^{k},\bbr_{n,n_{r,k}}^{k+1}}+\rho\| \boldsymbol{\theta}_{n_{r,k}}^{k+1} -  \boldsymbol{\theta}_{n_{r,k}}^k\|^2+\rho\|\boldsymbol{\theta}_{n_{r,k}}^{k+1} -  \boldsymbol{\theta}^\star\|^2-\rho\|\boldsymbol{\theta}_{n_{r,k}}^{k} -  \boldsymbol{\theta}^\star\|^2\Big).
\end{align} 
Substituting the equalities from \eqref{term_1_ap11} and \eqref{term_2_ap11} to the left hand side of \eqref{eq7a11}, we obtain

\begin{align}\label{new_211}
\sum_{n=1}^{N-1}&\Big[-(1/\rho)\|\lamb_n^{k+1}-\lamb_n^\star\|^2+ (1/\rho)\|\lamb_n^{k}-\lamb_n^\star\|^2-\rho\|\bbr_{n,n_{r,k}}^{k+1}\|^2\Big]\nonumber
\\
&-\!\!\!\!\!\!\!\!\sum_{n\in {\cal N}_h^k\setminus\{1\}}\!\!\!\Big(\!\!\!-2\rho\ip{\boldsymbol{\theta}_{n_{l,k}}^{k+1}-\boldsymbol{\theta}_{n_{l,k}}^{k},\bbr_{n_{l,k},n}^{k+1}}\!+\!\rho\|\boldsymbol{\theta}_{n_{l,k}}^{k+1} -  \boldsymbol{\theta}_{n_{l,k}}^k\|^2
\!+\!\rho\|\boldsymbol{\theta}_{n_{l,k}}^{k+1} \!-\!  \boldsymbol{\theta}^\star\|^2\!\!-\!\rho\| \boldsymbol{\theta}_{n_{l,k}}^{k} -  \boldsymbol{\theta}^\star\|^2\Big)\nonumber
\\
&-\!\!\!\sum_{n\in {\cal N}_h^k}\Big(2\rho\ip{\boldsymbol{\theta}_{n_{r,k}}^{k+1}-\boldsymbol{\theta}_{n_{r,k}}^{k},\bbr_{n,n_{r,k}}^{k+1}}+\rho\|\boldsymbol{\theta}_{n_{r,k}}^{k+1} -  \boldsymbol{\theta}_{n_{r,k}}^k\|^2+\rho\|\boldsymbol{\theta}_{n_{r,k}}^{k+1} -  \boldsymbol{\theta}^\star\|^2-\rho\|\boldsymbol{\theta}_{n_{r,k}}^{k} -  \boldsymbol{\theta}^\star\|^2\Big)\nonumber 
\\
&\geq 0,
\end{align}
Next, consider the Lyapunov function $V_k$ as 
\begin{align}
V_k =1/\rho\sum_{n=1}^{N-1}\|\lamb_n^{k}-\lamb_n^\star\|^2 +\rho\sum_{n\in {\cal N}_h^k\setminus\{1\}}\|\boldsymbol{\theta}_{n_{l,k}}^{k} -  \boldsymbol{\theta}^\star\|^2+\rho\sum_{n\in {\cal N}_h^k}\|\boldsymbol{\theta}_{n_{l,k}}^{k} -  \boldsymbol{\theta}^\star\|^2.
\label{lyapEq11}
\end{align}
After rearranging the terms in \eqref{new_211} and using the definition of the Lyapunov function in \eqref{lyapEq11}, we get

\begin{align}
V_{k+1}\leq V_k-&
\sum_{n=1}^{N-1}\rho \|\bbr_{n,n_{r,k}}^{k+1}\|^2- \Big[\sum_{n\in {\cal N}_h^k\setminus\{1\}} \rho\| \boldsymbol{\theta}_{n_{l,k}}^{k+1} -  \boldsymbol{\theta}_{n_{l,k}}^k\|^2+\sum_{n\in {\cal N}_h^k} \rho\| \boldsymbol{\theta}_{n_{r,k}}^{k+1} -  \boldsymbol{\theta}_{n_{r,k}}^k\|^2\Big]\nonumber
\\
 - &\Big[\sum_{n\in {\cal N}_h^k\setminus\{1\}}-2\rho\ip{\boldsymbol{\theta}_{n_{l,k}}^{k+1}-\boldsymbol{\theta}_{n_{l,k}}^{k},\bbr_{n_{l,k},n}^{k+1}}+\sum_{n\in {\cal N}_h^k}2\rho\ip{\boldsymbol{\theta}_{n_{r,k}}^{k+1}-\boldsymbol{\theta}_{n_{r,k}}^{k},\bbr_{n,n_{r,k}}^{k+1}}\Big].
\label{mainIneq11}
\end{align}
We rewrite \eqref{mainIneq11} as
\begin{align}
V_{k+1}\leq V_k-&
\sum_{n\in {\cal N}_h^k\setminus\{1\}}\rho \|\bbr_{n_{l},n}^{k+1}\|^2+\sum_{n\in {\cal N}_h^k}\rho \|\bbr_{n,n_{r}}^{k+1}\|^2- \Big[\sum_{n\in {\cal N}_h^k\setminus\{1\}} \rho\| \boldsymbol{\theta}_{n_{l,k}}^{k+1} -  \boldsymbol{\theta}_{n_{l,k}}^k\|^2+\sum_{n\in {\cal N}_h^k} \rho\| \boldsymbol{\theta}_{n_{r,k}}^{k+1} -  \boldsymbol{\theta}_{n_{r,k}}^k\|^2\Big]\nonumber
\\
 - &\Big[\sum_{n\in {\cal N}_h^k\setminus\{1\}}-2\rho\ip{\boldsymbol{\theta}_{n_{l,k}}^{k+1}-\boldsymbol{\theta}_{n_{l,k}}^{k},\bbr_{n_{l,k},n}^{k+1}}+\sum_{n\in {\cal N}_h^k}2\rho\ip{\boldsymbol{\theta}_{n_{r,k}}^{k+1}-\boldsymbol{\theta}_{n_{r,k}}^{k},\bbr_{n,n_{r,k}}^{k+1}}\Big].
\label{mainIneq11_211}
\end{align}
Next, the equation in \eqref{mainIneq11_211} can be re-written as
\begin{align}
V_{k+1}\leq V_k-&
\rho\sum_{n\in {\cal N}_h^k\setminus\{1\}}\Big[ \|\bbr_{n_{l},n}^{k+1}\|^2-2\ip{\boldsymbol{\theta}_{n_{l,k}}^{k+1}-\boldsymbol{\theta}_{n_{l,k}}^{k},\bbr_{n_{l,k},n}^{k+1}}+\|\boldsymbol{\theta}_{n_{l,k}}^{k+1} -  \boldsymbol{\theta}_{n_{l,k}}^k\|^2\Big]\nonumber\\
&-\rho\sum_{n\in {\cal N}_h^k}\Big[\|\bbr_{n,n_{r}}^{k+1}\|^2+2\ip{\boldsymbol{\theta}_{n_{r,k}}^{k+1}-\boldsymbol{\theta}_{n_{r,k}}^{k},\bbr_{n,n_{r,k}}^{k+1}}+\|\boldsymbol{\theta}_{n_{r,k}}^{k+1} -  \boldsymbol{\theta}_{n_{r,k}}^k\|^2\Big]
\label{mainIneq11_2}
\end{align}
Further, we write \eqref{mainIneq11_2} as
\begin{align}
V_{k+1}\leq V_k-&
\rho\left(\sum_{n\in {\cal N}_h^k\setminus\{1\}}\|\bbr_{n_{l},n}^{k+1}-(\boldsymbol{\theta}_{n_{l,k}}^{k+1} -  \boldsymbol{\theta}_{n_{l,k}}^k)\|^2
+\!\!\!\!\!\sum_{n\in {\cal N}_h^k}\|\bbr_{n,n_{r}}^{k+1}+(\boldsymbol{\theta}_{n_{r,k}}^{k+1} -  \boldsymbol{\theta}_{n_{r,k}}^k)\|^2\right)
\label{mainIneq11_311}
\end{align}

The result in \eqref{mainIneq11_311} proves that $V_{k
+1}$ decreases in each iteration $k$. Now, since $V_k\geq 0$ and $V_k\leq V_0$, it holds that $\Bigg[\sum_{n\in {\cal N}_h^k\setminus\{1\}}\|\bbr_{n_{l},n}^{k+1}-(\boldsymbol{\theta}_{n_{l,k}}^{k+1} -  \boldsymbol{\theta}_{n_{l,k}}^k)\|^2
+\!\!\!\!\!\sum_{n\in {\cal N}_h^k}\|\bbr_{n,n_{r}}^{k+1}+(\boldsymbol{\theta}_{n_{r,k}}^{k+1} -  \boldsymbol{\theta}_{n_{r,k}}^k)\|^2\Bigg]$ is bounded. Taking the telescopic sum over $k$ in \eqref{mainIneq11_311} and taking limit $K\rightarrow\infty$, we get 
\begin{align}
\lim_{K\rightarrow\infty}\sum\limits_{k=0}^{K}\Bigg[\sum_{n\in {\cal N}_h^k\setminus\{1\}}\|\bbr_{n_{l},n}^{k+1}-(\boldsymbol{\theta}_{n_{l,k}}^{k+1} -  \boldsymbol{\theta}_{n_{l,k}}^k)\|^2
+\!\!\!\!\!\sum_{n\in {\cal N}_h^k}\|\bbr_{n,n_{r}}^{k+1}+(\boldsymbol{\theta}_{n_{r,k}}^{k+1} -  \boldsymbol{\theta}_{n_{r,k}}^k)\|^2\Bigg] \leq V_0.
\label{mainIneq311}
\end{align}
The result in \eqref{mainIneq311} implies that  the primal residual $\bbr_{n,n_{r,k}}^{k+1}\rightarrow \boldsymbol{0}$ as $k\rightarrow\infty$  for all $n \in \{1,\cdots, N-1\}$.  Similarly, the norm differences $\norm{ \boldsymbol{\theta}_{n_{l,k}}^{k+1} -  \boldsymbol{\theta}_{n_{l,k}}^k}$ and $\norm{ \boldsymbol{\theta}_{n_{r,k}}^{k+1} -  \boldsymbol{\theta}_{n_{r,k}}^k}\rightarrow\boldsymbol{0}$ as $k\rightarrow\infty$ which implies that the dual residual $\bbs_{n}^{k}\rightarrow\boldsymbol{0}$ as $k\rightarrow \infty$ for all $n\in {\cal N}_h^k$.  In order to prove the convergence to optimal point, , consider the lower and the upper bounds on the objective function optimality gap given by 
\begin{align}
\sum_{n=1}^N[f_n(\boldsymbol{\theta}_{n}^{k+1})-f_n(\boldsymbol{\theta}^\star)]&\leq -\sum_{n=1}^{N-1}\ip{\lamb_{n}^{k+1},\bbr_{n,n_{r,k}}^{k+1}}+\sum_{n\in {\cal N}_h^k}\ip{\bbs_{n}^{k+1},\boldsymbol{\theta}^\star-\boldsymbol{\theta}_n^{k+1}}
\label{lem1Eq2311}\\
\sum_{n=1}^N[f_n(\boldsymbol{\theta}_{n}^{k+1})-f_n(\boldsymbol{\theta}^\star)]&\geq -\sum_{n=1}^{N-1}\ip{\lamb_{n}^\star,\bbr_{n,n_{r,k}}^{k+1}}
	\label{eq3a311}.
\end{align}
Note that from the results established in this appendix, it holds that the right hand side of the upper bound in \eqref{lem1Eq2311} converge to zero as $k\rightarrow\infty$ and also the right hand side of the lower bound in \eqref{eq3a311} converges to zero as $k\rightarrow\infty$. This implies that 
\begin{align}
\lim_{k\rightarrow\infty}\sum_{n=1}^N[f_n(\boldsymbol{\theta}_{n}^{k+1})-f_n(\boldsymbol{\theta}^\star)] =0
\end{align}
which is the required result. Hence proved.

\bibliography{main}

\end{document}